\documentclass{article}

\usepackage{times}
\usepackage{fullpage}
\usepackage[numbers, compress]{natbib}
\usepackage[utf8]{inputenc} 
\usepackage[T1]{fontenc}    
\usepackage{hyperref}       
\usepackage{url}            
\usepackage{booktabs}       
\usepackage{nicefrac}       
\usepackage{microtype}      
\usepackage{wrapfig,amsmath,amsfonts,amsthm, amssymb, bm, color, enumitem,algorithm, algorithmic}
\usepackage{xfrac}
\usepackage{mathrsfs}
\usepackage{color}
\usepackage{multicol}

\newcommand{\ip}[2]{\left\langle #1, #2 \right\rangle}

\renewcommand{\Pr}[1]{\operatorname{Pr}\left[{#1}\right]}

\newcommand{\norm}[1]{\left\lVert{#1}\right\rVert}
\newcommand{\frob}[1]{\left\lVert{#1}\right\rVert_F}

\newcommand{\abs}[1]{\left\lvert{#1}\right\rvert}
\newcommand{\R}{\mathbb{R}}
\newcommand{\Rnk}{\mathbb{R}^{n\times k}}
\newcommand{\Rn}{\mathbb{R}^{n}}
\newcommand{\Rk}{\mathbb{R}^{k}}
\newcommand{\Rm}{\mathbb{R}^{m}}
\newcommand{\Rnn}{\mathbb{R}^{n\times n}}

\newcommand{\N}{\mathcal{N}}

\newcommand{\trans}[1]{ {#1}^{\!\top}}
\newcommand{\trace}[1]{\mathrm{Tr}\left(#1\right)}
\newcommand{\Sym}[1]{\mathcal{S}^{#1}}
\newcommand{\Snn}{\Sym{n\times n}}

\newcommand{\nulll}{\operatorname{null}}

\newtheorem{theorem}{Theorem}
\newtheorem{lemma}{Lemma}
\newtheorem*{lemma*}{Lemma}
\newtheorem{corollary}{Corollary}

\newtheorem{definition}{Definition}
\newtheorem{proposition}{Proposition}

\newcommand{\defeq}{\triangleq}

\newcommand{\eps}{\epsilon}
\newcommand{\calA}{\mathcal{A}}
\newcommand{\calC}{\mathcal{C}}
\newcommand{\calN}{\mathcal{N}}
\newcommand{\calM}{\mathcal{M}}
\newcommand{\rank}{\operatorname{rank}}
\newcommand{\im}{\operatorname{im}}

\DeclareMathOperator*{\minimize}{\mathrm{minimize}}

\newcommand{\Fm}{F_{\mu}}
\newcommand{\sG}{\sigma_G}
\newcommand{\tG}{G}
\newcommand{\calS}{\mathcal{S}}
\newcommand{\calE}{\mathcal{E}}
\newcommand{\vr}{\mathbf{r}}
\newcommand{\vb}{\mathbf{b}}
\newcommand{\const}{c_0}

\newcommand{\bz}{b_0}

\newtheorem{assumption}[theorem]{Assumption}

\title{Smoothed analysis for low-rank solutions to semidefinite programs in quadratic penalty form} 

\author{Srinadh Bhojanapalli\thanks{TTI Chicago, email: srinadh@ttic.edu } \and Nicolas Boumal\thanks{Princeton University, email: nboumal@math.princeton.edu} \and Prateek Jain\thanks{Microsoft Research, email: prajain@microsoft.com} \and Praneeth Netrapalli \thanks{Microsoft Research, email: praneeth@microsoft.com}}

\date{}

\begin{document}

\maketitle

\begin{abstract}
Semidefinite programs (SDP) are important in learning and combinatorial optimization with numerous applications. In pursuit of low-rank solutions and low complexity algorithms, we consider the Burer--Monteiro factorization approach for solving SDPs. We show that all approximate local optima are global optima for the penalty formulation of appropriately rank-constrained SDPs as long as the number of constraints scales sub-quadratically with the desired rank of the optimal solution. Our result is based on a simple penalty function formulation of the rank-constrained SDP along with a smoothed analysis to avoid worst-case cost matrices. We particularize our results to two applications, namely, Max-Cut and matrix completion.

\end{abstract}


\section{Introduction}
\label{sec:intro}

Semidefinite programs (SDP) are an important class of optimization problems \citep{vandenberghe1996semidefinite}, and are critical to several learning-related tasks, e.g., clustering \citep{shi2000normalized,abbe2017community}, matrix completion and regression \citep{recht2010guaranteed,candes2009exact}, kernel learning \citep{lanckriet2004learning}, sum-of-squares relaxations \citep{barak2015dictionary}, etc. 

However, solving SDPs in practice is a challenging task. Consider the following  canonical SDP: 
\begin{align}
\underset{X \in \R^{n \times n}}{\minimize} \quad & \ip{C}{X} \nonumber \\
\text{subject to} \quad & \ip{A_i}{X} =b_i,\quad i =1,\cdots, m, \text{ and } X \succeq 0,
\label{eq:sdp}
\end{align}
where $C, A_1, \ldots, A_m \in \R^{n\times n}$ are symmetric matrices, $\ip{A}{B} = \trace{A^T B}$, and $X$ is positive semidefinite. Such problems are convex and can be solved in polynomial time using classical iterative algorithms such as ellipsoid and interior-point methods~\citep{nesterov1994interior}.  However, these algorithms have super-linear complexity (in input size) and tend to scale poorly in practice, and are not well suited for typical learning tasks where both $m$ and $n$ can be fairly large. The two key challenges for these algorithms are: (a) a search space of high dimension on the order of $n^2$; and (b) the need to maintain positive semidefiniteness of the variable matrix $X$ throughout the iterations. 

%
%

In response to these challenges, \citet{burer2003nonlinear, burer2005local} suggested solving~\eqref{eq:sdp} by constraining the search space to matrices of rank at most $k$, using a parameterization of the form $X = UU^T$ where $U \in \Rnk$. This reduces the number of variables from $O(n^2)$ to $O(nk)$, and mechanically enforces positive semidefiniteness:
\begin{align}
	\underset{U \in \Rnk}{\minimize} \quad & \ip{C}{UU^T} \nonumber \\
	\text{subject to} \quad & \ip{A_i}{UU^T} = b_i,\quad i = 1,\cdots, m.
	\label{eq:factored}
\end{align}
This is equivalent to~\eqref{eq:sdp} with the additional constraint $\rank(X) \leq k$. This rank constraint is fairly natural, as several SDPs of interest are themselves relaxations of rank-constrained problems. Moreover, \citet{barvinok1995problems, pataki1998rank} showed that for every compact SDP with a solution, there exists a rank $\Omega(\sqrt{m})$ solution that is also globally optimal. While this ensures that the global optimum of the factored SDP problem (with $k=\Omega(\sqrt{m})$) is a global optimum of the original SDP problem, it is not immediately clear how to solve the factorized problem. 

In fact, the factorized problem is a non-convex quadratically constrained quadratic program which in general can be NP-hard. The challenge in solving the problem arises due to the non-convexity as well as due to constraints. In this work, we propose a simple penalty method that gets rid of the constraints and replaces them via a quadratic penalty in the objective function. The penalty formulation allows us to study first-order and second-order stationary points of the problem and lends itself to efficient algorithms, as we detail in this paper.

The proposed penalty formulation is given by: 
\begin{align}
	\underset{X \succeq 0}{\minimize} ~\quad \Fm(X) =  \ip{C}{X} + \mu \sum_{i=1}^m \left(\ip{A_i}{X} -b_i\right)^2,
	\label{eq:penalty_sdp}
\end{align}
where $\mu$ is generally a large positive constant. Notice that this is a convex problem. Intuitively, for increasingly large $\mu$, solutions of~\eqref{eq:penalty_sdp} converge to solutions of~\eqref{eq:sdp}.

Combining the formulation with the Burer--Monteiro factorization we get:
\begin{align}
	\underset{U \in \Rnk}{\minimize} \quad  L_{\mu}(U) =  \ip{C}{UU^T} + \mu \sum_{i=1}^m (\ip{A_i}{UU^T} -b_i)^2.
	\label{eq:penalty_factored}
\end{align}
The cost function $L_{\mu}$ is non-convex, and generic optimization algorithms can only guarantee computation of an approximate second-order stationary points (SOSP) \citep{cartis2012complexity,ge2015escaping}. That is, such algorithms converge to a point $U$ where the gradient of $L_\mu$ is small and the Hessian of $L_\mu$ is almost positive semidefinite. Such second-order stationary points need not be close to optimal in general.

We construct an explicit SDP where a suboptimal SOSP exists even for $k$ as large as $n-1$. However, we show that there are only measure zero of such bad SDPs. Hence, we show that if the cost matrix has a small amount of randomness then {\em any} SOSP of $L_\mu$ is a global optimum, as long as at least one SOSP exists. That is, for almost all cost matrices $C$, an SOSP of \eqref{eq:penalty_factored} corresponds to a global optimum. We would like to stress here that for certain non-compact SDPs existence of an SOSP itself is not guaranteed. However, as shown in Section~\ref{sec:applications}, SOSPs exists for several important SDPs. 

We next address the question of approximate optimality for approximate SOSPs, as optimization algorithms can only recover approximate SOSPs in polynomial time. Since there is a measure zero set of SDPs with bad SOSPs, there can be a non-zero (but small) measure set of SDPs with bad approximate SOSPs. We use {\em smoothed analysis} to avoid these bad SDPs, by perturbing the objective matrix. We show that for $k=\tilde \Omega(\sqrt{m})$, any approximate SOSP of  $L_{\mu}$  with a perturbed objective and bounded residues is approximately optimal to the penalty objective~\eqref{eq:penalty_sdp}. We further discuss settings under which all SOSPs of the penalty objective have bounded solutions (residues).

Since our results are about approximate SOSPs, and not any particular algorithm, it is an interesting question to see if we can adapt classical techniques such as interior point or cutting plane to optimize over the low dimensional, factored space. We provide results for gradient descent convergence in Section \ref{sec:gd}.

Finally, even though finding the smallest rank solution satisfying a set of linear equations is NP hard \citep{natarajan1995sparse}, our result shows how increasing the number of parameters (rank) makes the optimization of this non-convex problem easier. While the extreme case of rank $n$ makes the constraint trivial, our results show optimality for a non-trivial rank ($\tilde{\Omega}(\sqrt{m})$), and it is an interesting question to understand this trade-off in more detail.

\subsection{Main results}
The main contributions of this work are: 
\begin{itemize}
	\item We propose a simple penalty version of the factored SDP~\eqref{eq:factored} and show that, for almost all cost matrices $C$, any exact SOSP of the rank-constrained formulation~\eqref{eq:penalty_factored} is a global optimum for rank ${\Omega}(\sqrt{m})$---see Corollary~\ref{cor:exactpenaltyfactorized}. This result removes the smooth manifold requirement of~\citep{boumal2016non}, though it applies to~\eqref{eq:penalty_sdp}, not~\eqref{eq:sdp}.
	\item We show that there indeed exists a compact, feasible SDP with a worst-case $C$ for which the penalized, factorized problem admits a suboptimal SOSP (see Theorem~\ref{thm:bad_sdp}), even for rank almost as big as the dimension.
	\item We show that by perturbing the objective function slightly and by performing a smoothed analysis on the resulting problem, we can guarantee every approximate SOSP of the perturbed problem is an approximate global optimum of the perturbed and penalized SDP. Hence, we can use standard techniques~\citep{cartis2012complexity,ge2015escaping} to find approximate SOSPs and guarantee global optimality---see Theorem~\ref{thm:optimal_approx_compact}.
\end{itemize}

In summary, we show that for a class of SDPs with bounded solutions, we can find a low-rank solution that is close to the global optimum of the penalty objective. We believe that the factorization technique can be leveraged to design faster SDP solvers, and any looseness in the current bounds is an artifact of our proof, which hopefully can be tightened in future works. 

\subsection{Prior work}

Fast solvers for SDPs have garnered interest in the optimization and in the theoretical computer science communities for a long time. Most of the existing results for SDP solvers can be categorized into direct (convex) methods and factorization methods. \\

\noindent \textbf{Convex methods:}
Classical techniques such as interior point methods~\citep{ nesterov1989self, nesterov1988polynomial, alizadeh1995interior} and cutting plane methods~\citep{anstreicher2000volumetric,krishnan2003properties} enjoy geometric convergence, but their computational complexity per iteration is high. As a result, it is hard to scale these methods to SDPs with a large number of variables.

With the goal of speeding up the computation, many works have considered: i) a specific and important class of SDPs, namely, SDPs with a trace constraint ($\trace{X}=1$), and ii) methods with sub-linear convergence. For these SDPs, \citet{arora2005fast} proposed a multiplicative weights method which provides faster techniques for some graph problems, with running time depending on $O(\frac{1}{\eps^2})$ and the \textit{width} of the problem.  \citet{hazan2008sparse} proposed a Frank--Wolfe-type algorithm with a complexity of $\tilde{O}( \frac{Z}{\eps^{3.5}})$ where $Z$ is the sparsity of $C$ and the $A_i$'s.  \citet{garber2016sublinear, garber2016faster}  proposed faster methods that either remove the dependence on $Z$ (sub-linear time), or improve the dependence on $\eps$. While these methods improve the per iteration complexity, they still need significant memory as the rank of solutions for these methods is not bounded, and scales at least at the rate of $O(\frac{1}{\eps})$.  An exception to this is the work by \citet{yurtsever2017sketchy}, which uses sketching techniques in combination with conditional gradient method to maintain a low rank representation. However this method is guaranteed to find a low rank optimum only if the conditional gradient method converges to a low rank solution. \\

\noindent \textbf{Factorization methods:}
\citet{burer2003nonlinear, burer2005local} proposed a different approach to speed up computations, namely by searching for solutions with smaller rank.  Even though all feasible compact SDPs have at least one solution of rank $O(\sqrt{m})$ \citep{barvinok1995problems, pataki1998rank}, it is not an easy task to optimize directly on the rank-constrained space because of non-convexity. However,  \citet{burer2003nonlinear, burer2005local} showed that any \emph{rank-deficient local minimum} is optimal for the SDP; \citet{journee2010low} extended this result to any rank-deficient SOSP under restrictive conditions on the SDP. However, these results cannot guarantee that SOSPs are rank deficient, or at least that rank-deficient SOSPs can be computed efficiently (or even exist). \citet{boumal2016non} address this issue by showing that for a particular class of SDPs satisfying some regularity conditions, and for almost all cost matrices $C$, any SOSP of the rank-constrained problem with $k = \Omega(\sqrt{m})$ is a global optimum. Later, \citet{pmlr-v65-mei17a} showed that for SDPs with elliptic constraints (similar to the Max-Cut SDP), any rank-$k$ SOSP gives a $(1-\frac{1}{k-1})$ approximation to the optimum value. Both these results are specific to particular classes of SDPs and do not extend to general problems.
 
In a related setup, \citet{keshavan2010matrix, jain2013low} have showed that rank-constrained matrix completion problems can be solved using smart initialization strategies followed by local search methods. Following this, many works have identified interesting statistical conditions under which certain rank-constrained matrix problems have no spurious local minima \citep{ sun2016geometric, bandeira2016low, ge2016matrix, bhojanapalli2016global, park2017non, ge2017no, zhu2017global, ge2017optimization}.  These results are again for specific problems and do not extend to general SDPs.

In contrast, our result holds for a large class of SDPs in penalty form, without strong assumptions on the constraint matrices $A_i$ and for a large class of cost matrices $C$. We avoid degenerate SDPs with spurious local minima by perturbing the problem and then using a smoothed analysis, which is one of the main contribution of the work.


\subsection*{Notation}

For a smooth function $f(X)$, we refer to first-order stationary points $X$ as FOSPs. Such points satisfy $\nabla f(X) = 0$ (zero gradient). We refer to second-order stationary points at SOSPs. Such points are FOSPs and furthermore satisfy $\nabla^2 f(X) \succeq 0$, i.e., the Hessian is positive semidefinite. The set of symmetric matrices of size $n$ is $\Snn$. $\sigma_i()$  and $\lambda_i()$ denote the $i$th singular- and eigenvalues respectively, in decreasing order.

\section{Exact second-order points typically are optimal}
\label{sec:exact}
In this section, we study the second-order stationary points of our penalty formulation \eqref{eq:penalty_factored} and show that for ``typical'' cost matrices $C$, {\em exact} SOSPs are optimal for \eqref{eq:penalty_factored} as long as $k={\Omega}(\sqrt{m})$. 

Our result is based on 	a simple but powerful argument that has appeared in various forms before, notably in \citep{burer2005local}. The argument claims that any {\em rank-deficient} local optimum of \eqref{eq:penalty_factored} (which is really a parameterized version of~\eqref{eq:penalty_sdp} with a rank constraint) should map to a local optimum of \eqref{eq:penalty_sdp} as the constraint $\rank(X) \leq k$ is not active. Since~\eqref{eq:penalty_sdp} is convex, every local optimum is a global optimum, hence a rank-deficient local optimum of \eqref{eq:penalty_factored} maps to a global optimum of \eqref{eq:penalty_sdp}. Interestingly, the result holds even if $U$ is just an SOSP rather than a local optimum, something that is readily apparent from the proofs in~\citep{journee2010low}, albeit in a restricted setting.
\begin{lemma}\label{lem:global}
	Let $f(X)$ be a convex, twice continuously differentiable function of $X \in \Snn$. Consider the convex problem
	\begin{align}
		\underset{X \succeq 0}{\minimize}~ f(X).
		\label{eq:prob_fx}
	\end{align}
	Now consider the rank-constrained factorized version of the problem:
	\begin{align}
		\underset{U \in \R^{n \times k}}{\minimize} ~ g(U) = f(UU^T).
		\label{eq:prob_fx_U}
	\end{align}
	If $U$ is an SOSP of~\eqref{eq:prob_fx_U} with $\rank(U) < k$, then $U$ is a global minimum of~\eqref{eq:prob_fx_U} and $UU^T$ is a global minimum of~\eqref{eq:prob_fx}. (Notice that such a point may not exist in general.)
\end{lemma}
See Appendix~\ref{app:exact} for a detailed proof. 

Thus, (column) rank-deficient SOSPs of~\eqref{eq:penalty_factored} are globally optimal and map to global optima of~\eqref{eq:penalty_sdp}. A direct corollary states that non-convexity is benign if $k = n+1$.
\begin{corollary}
	Given an SDP in penalized and factorized form~\eqref{eq:penalty_factored} with $k > n$, for almost any cost matrix $C$, deterministically, any SOSP $U$ is a global optimum, and $UU^T$ is a global optimum for~\eqref{eq:penalty_sdp}.
\end{corollary}
Yet, the main goal is to make a statement for small $k$, so as to reduce the dimensionality of the search space. Unfortunately, in general, SOSPs of non-convex cost functions need not have rank less than $k$ for arbitrary $k$. 

However, the following lemma asserts that, for almost all cost matrices $C$, provided $k$ grows like $\sqrt{m}$, all FOSPs (a fortiori, all SOSPs) are rank deficient. Our proof is the same as that of~\citep[Lemma 9]{boumal2016non} but the main statement as well as the cost function and conditions on constraints are different. In particular, unlike~\cite{boumal2016non}, we do not require that the feasible set of~\eqref{eq:factored} form a smooth manifold. 
\begin{lemma}\label{lem:rank_deficient}
	Choose $k$ such that $\frac{k(k+1)}{2} > m$. For almost any $C \in \Snn$, any FOSP $U \in \Rnk$ of~\eqref{eq:penalty_factored} (if one exists) satisfies $\rank(U) < k$.
\end{lemma}
See Appendix~\ref{app:exact} for a detailed proof. 

These two lemmas lead to an important corollary regarding the factorization approach.
\begin{corollary} \label{cor:exactpenaltyfactorized}
	Given an SDP in penalized and factorized form~\eqref{eq:penalty_factored} with $k$ such that $\frac{k(k+1)}{2} > m$, for almost any cost matrix $C$, deterministically, any SOSP $U$ is a global optimum, and $UU^T$ is a global optimum for~\eqref{eq:penalty_sdp}.
\end{corollary}
To ensure existence of such solutions, it is necessary to include additional conditions (for example, on the constraints of the SDP.) From~\citep{pataki1998rank,barvinok1995problems}, it is known that SDPs with non-empty, compact search spaces can have a unique solution of rank up to the maximal $k$ such that $\frac{k(k+1)}{2} \leq m$. This indicates that, in general, the condition on $k$ cannot be improved. 

These observations lead to the following two natural questions: 
\begin{enumerate}
	\item Our result holds only for ``typical'' $C$. Is this an artifact of our proof technique, or is it necessary to exclude a zero-measure set of cost matrices $C$?
	\item Our result holds only for \emph{exact} SOSPs, which in general are hard to compute. Numerical methods tend to provide approximate SOSPs only. Can we extend the results to \emph{approximate} SOSPs as well?
\end{enumerate}
The next section answers the first question in the affirmative: there do exist ``bad'' matrices $C$ for some SDPs, so that any result of the type of Corollary~\ref{cor:exactpenaltyfactorized} must exclude at least some SDPs. To address the second question, we resort to smoothed analysis, that is, for large classes of SDPs in penalty form, upon perturbing the cost matrix randomly, we show that approximate SOSPs are also good enough to obtain approximately globally optimal solutions of the perturbed problem. 

\section{Exact second-order points sometimes are suboptimal}
\label{sec:exact_sub}

Below, we construct an SDP which confirms that it is indeed necessary (in full generality) to exclude some SDPs in Corollary~\ref{cor:exactpenaltyfactorized}, even if $k$ is allowed to grow large.

Pick $n \geq 3$ and set $\eps = \sqrt{\frac{6}{n-1}}$. Consider the following $m = n+1$ constraint matrices in $\Snn$:
\begin{align*}
A_i & = e_i \trans{e_{n}} + e_{n} \trans{e_{i}}, \quad i = 1, \cdots, n-1, \\
A_{n} & = \epsilon \begin{bmatrix}
I_{n-1} & 0 \\ 0 & 1
\end{bmatrix}, \quad \mbox{ and } \quad 
A_{n+1} = \epsilon \begin{bmatrix}
2I_{n-1} & 0 \\ 0 & 1
\end{bmatrix},
\end{align*}
where $e_i \in \R^n$ is the $i$th standard basis vector (the $i$th column of $I_n$).
In words, for $i = 1, \ldots, n-1$, each $A_i$ has only two non-zero entries---both equal to one---located in row $i$ of the last column and symmetrically in column $i$ of the last row.
Pair these matrices with the right-hand side vector $b \in \R^m$ defined by
\begin{align*}
	b_1 = \cdots = b_{n-1} = 0, \quad \textrm{ and } \quad  b_{n} = b_{n+1} = \eps \frac{5(n-1)}{3}.
\end{align*}
Finally, set the cost matrix $C$ to be zero. (A distinct advantage of picking $C = 0$ is that it makes the choice of $\mu > 0$ irrelevant in defining~\eqref{eq:penalty_factored}.) These prescriptions fully define the SDP~\eqref{eq:sdp} and its associated factorized and penalized problem~\eqref{eq:penalty_factored}, which we can write here as:
\begin{align}
	\underset{U \in \R^{n\times k}}{\minimize} \quad L(U) = \frac{1}{2} \sum_{i=1}^{n+1} \left(\ip{A_i}{UU^T} - b_i\right)^2. \label{eqn:npsdp2}
\end{align}
\begin{theorem}\label{thm:bad_sdp}
	The SDP defined above admits a global optimum of rank 1. Furthermore, for $k = n-1$, $\bar U \defeq \begin{bmatrix}I_{n-1} \\ 0 \end{bmatrix}$ is a suboptimal SOSP of~\eqref{eqn:npsdp2}.
\end{theorem}
See Appendix~\ref{app:exact_sub} for a detailed proof of the theorem. 

\section{Approximate second-order points: smoothed analysis}
\label{sec:smooth}

Recall that Corollary~\ref{cor:exactpenaltyfactorized} shows that {\em exact} SOSPs of \eqref{eq:penalty_factored} are optimal for almost all cost matrices $C$. However, obtaining exact SOSPs is challenging in practice. Standard optimization algorithms such as the trust-region method and the cubic regularization method~\citep{nesterov2006cubic,cartis2012complexity}, when run for finitely many iterations, converge to an {\em approximate} SOSP only, as defined below. 
All proofs for this section are in Appendix~\ref{app:smooth}.
\begin{definition}[$\eps$-FOSP]
	We call $U$ an $\eps$-FOSP of a function $f(U)$ if:
	\begin{align*}
		\| \nabla f(U) \|_F \leq \eps.
	\end{align*}
\end{definition}
\begin{definition}[$(\eps, \gamma)$-SOSP]
	We call $U$ an $(\eps, \gamma)$-SOSP of a function $f(U)$ if:
	\begin{align*}
		\| \nabla f(U) \|_F \leq \eps ~~ \text{and} ~~ \lambda_{\min} (\nabla^2 f(U)) \geq -\gamma \sqrt{\eps}.
	\end{align*}
\end{definition}
As an extension to Lemma~\ref{lem:global}---which states rank-deficient \emph{exact} SOSPs are optimal---we now show that approximate SOSPs which are also approximately rank deficient are indeed approximately optimal. 
To this end, we define the linear operator $\calA \colon \Snn \to \Rm$ with $\calA(X)_i = \ip{A_i}{X}$. We use the following notion of norm for $\calA$:
\begin{align}
	\|\calA\| & \triangleq \max_{Y \in \Snn, \|Y\|_F \leq 1} \|\calA(Y)\|_2, & \|\calA\| = \|\calA^*\| & \triangleq \max_{y \in \Rm, \|y\|_2 \leq 1} \|\calA^*(y)\|_F.
	\label{eq:normofA}
\end{align}
Furthermore, we define the residue at a point $U$ to be the vector of constraint violations:
\begin{align}
	\vr & = \vr(U) = \vr(UU^T) \triangleq \calA(UU^T) - \vb.
	\label{eq:residue}
\end{align}
\begin{lemma}\label{lem:compact_optimal_approx}
	Let $U \in \Rnk$ be an $(\eps, \gamma)$-SOSP of~\eqref{eq:penalty_factored} such that $\sigma_k^2(U) \leq \frac{\gamma \sqrt{\eps}}{8 \mu\|\calA\|^2}$. Then,
	\begin{align*}
		\lambda_{\min}(C + 2\mu\calA^*(\vr)) \geq -\gamma\sqrt{\eps}.
	\end{align*}
	Furthermore, if a global optimum $\tilde X$ for~\eqref{eq:penalty_sdp} exists, then the optimality gap obeys:
	\begin{align*}
		F_\mu(UU^T) - F_\mu(\tilde X) & \leq \gamma \sqrt{\eps} \operatorname{Tr}(\tilde X) + \frac{1}{2}\eps \|U\|_F.
	\end{align*}
	(Once again, we stress that $U$ and $\tilde X$ as prescribed may not exist.)
\end{lemma} 
To reach a statement about approximate optimality of approximate SOSPs, it remains to show that approximate FOSPs are approximately rank deficient. Such a result would constitute a generalization of Lemma~\ref{lem:rank_deficient}. In that lemma, we had to exclude a pathological set of ``bad'' matrices $C$. Hence, here too, we expect to encounter difficulties with some $C$'s.

For this reason, we resort to a \emph{smoothed analysis}. That is: on the off-chance that the cost matrix $C$ is ``bad'', we perturb it with a random Gaussian matrix. We further assume that (a) $k$ is large enough, and (b) approximate FOSPs have bounded residues $\vr$. That residues are indeed bounded is established under special conditions in later subsections.
\begin{theorem}\label{thm:compact_eps_fosp}
	Draw a random matrix $G$ with $G_{ij} \sim \calN(0, \sigma_G^2)$ i.i.d.\ for $i \leq j$ and $G = G^T$. Let $U \in \Rnk$ be an $\eps$-FOSP of~\eqref{eq:penalty_factored} with perturbed cost matrix $C+G$. Assume there exists a constant $B$ which only depends on the problem parameters $\calA, \vb, C$ and on $\eps, \mu$ such that:
	\begin{enumerate}
		\item With probability at least $1-\delta$ on the choice of $G$, all $\eps$-FOSPs of the perturbed problem have bounded residue: $\|\vr\|_2  \leq B$, and
		\item $k \geq 3\left[\log(\frac{n}{\delta'}) + \sqrt{ \rank(\calA)   \log\left( 1 + \frac{8 \mu B\|\calA\| \sqrt{\const  n}}{\sG } \right) }  \right]$ for some $\delta' \in (0, 1)$, where $\const$ is a universal constant.
	\end{enumerate}
	Then, with probability at least $1 - \delta - \delta'$,
	\begin{align*}
	\sigma_k(U) \leq \frac{ 2\epsilon}{\sG } \frac{\sqrt{\const n}}{k}.
	\end{align*}
\end{theorem}
Crucially, notice that $\rank(\calA) \leq m$, so that (up to log factors) $k$ is required to grow like $\sqrt{m}$, as desired.

\subsection{Compact SDPs} \label{sec:compact}

To leverage Lemma~\ref{lem:compact_optimal_approx} and Theorem~\ref{thm:compact_eps_fosp}, we must control the residues at approximate FOSPs of~\eqref{eq:penalty_factored}. This is delicate in general. In this part, we make the following assumption.
\begin{assumption}\label{assu:compactnonempty}
	The search space $\calC = \{ X \succeq 0 : \calA(X) = b \}$ of the SDP~\eqref{eq:sdp} is non-empty and compact, where $\calA \colon \Snn \to \R^m$ is the linear operator defined by $\calA(X)_i = \ip{A_i}{X}$.
\end{assumption}
When this is the case, standard results from~\citep{barvinok1995problems,pataki1998rank} guarantee the existence of a global optimum of rank $r$ where $\frac{r(r+1)}{2} \leq m$ for the SDP~\eqref{eq:sdp}---always. It is reasonable to expect such low-rank solutions might also exist for the penalized problem~\eqref{eq:penalty_sdp}, and that one should be able to compute these by solving the factorized problem~\eqref{eq:penalty_factored}---at least, generically. This section is about making these expectations precise in the soft case, where one only computes approximate SOSPs.

A technical necessity in our proofs is to show that FOSPs of~\eqref{eq:penalty_factored} have bounded norm. To do this, we need a technical modification of~\eqref{eq:penalty_factored}. Specifically, consider the following geometric fact.
\begin{proposition}\label{prop:compactSDPconsraintA0}
	For a given SDP~\eqref{eq:sdp}, assume $\calC$ is non-empty. Then, $\calC$ is compact if and only if there exists a positive definite matrix $A_0$ and a nonnegative real $b_0$ such that $\ip{A_0}{X} = b_0$ for all $X \in \calC$. Furthermore, unless $\calC = \{0\}$, $b_0 > 0$.
\end{proposition}
Thus, under Assumption~\ref{assu:compactnonempty}, we can rewrite~\eqref{eq:sdp} with an explicit redundant constraint involving $A_0 \succ 0$:
\begin{align}
	\underset{X \in \Snn}{\minimize} \quad & \ip{C}{X} \nonumber \\
	\text{subject to} \quad & \ip{A_0}{X} = b_0, \nonumber\\
							& \ip{A_i}{X} = b_i,\quad i = 1,\cdots, m, \textrm{ and } X \succeq 0.
	\label{eq:sdpcompact}
\end{align}
Accordingly, we define $\tilde \calA \colon \Snn \to \R^{m+1}$ and $\tilde \vb \in \R^{m+1}$ such that $\tilde \calA(X)_i = \ip{A_i}{X}$ for $i = 0, \ldots, m$, and $\calC = \{ X \succeq 0 : \tilde \calA(X) = \tilde \vb \}$. With the extended residue definition
\begin{align}
	\tilde\vr & = \tilde\vr(U) = \tilde\vr(UU^T) \triangleq \tilde \calA(UU^T) - \tilde \vb,
\end{align}
the associated penalty formulations are:
\begin{align}
	\underset{X \succeq 0}{\minimize} ~\quad \tilde\Fm(X) & =  \ip{C}{X} + \mu \|\tilde \vr(X)\|_2^2, \label{eq:penalty_sdp_compact} \\
	\underset{U \in \Rnk}{\minimize} \quad  \tilde L_{\mu}(U) & =  \ip{C}{UU^T} + \mu \|\tilde \vr(U)\|_2^2. \label{eq:penalty_factored_compact}
\end{align}
We note that, in full generality, finding $(A_0, b_0)$ as in Proposition~\ref{prop:compactSDPconsraintA0} may be as hard as solving an SDP, but in practical applications $(A_0, b_0)$ may be easy to determine. (For example, for the Max-Cut SDP,  feasible matrices have constant trace $n$, so that $A_0 = I_n$ and $b_0 = n$ are suitable.) More generally, SDPs with a trace constraint satisfy this with $A_0=I_n$. 

For this modified formulation, approximate FOSPs have bounded norm and bounded residues.
\begin{lemma}\label{lem:residues_compact}
	Consider problem~\eqref{eq:penalty_factored_compact} with $A_0 \succ 0$ and $b_0 \geq 0$. For any $U$,
	\begin{align*}
		\|\vr\|_2 & \leq \|\calA\| \|U\|_F^2 + \|\vb\|_2, & \textrm{ and } & & \|\tilde \vr\|_2 & \leq \|\tilde \calA\| \|U\|_F^2 + \|\tilde \vb\|_2.
	\end{align*}
	If $U$ is an $\eps$-FOSP and $b_0 > 0$, then
	\begin{small}
	\begin{align*}
		\|U\|_F^2 & \leq \max\left\{ \left(\frac{\eps}{2\mu \bz \lambda_{\max}(A_0)}\right)^2, \frac{1}{\lambda_{\min}(A_0)^2} \left(\frac{\|C\|_2}{2\mu} + \frac{3}{2}\bz \lambda_{\max}(A_0)\right) + \frac{\|\vb\|_2}{2\lambda_{\min}(A_0)} \right\}.
	\end{align*}\end{small}
\end{lemma}

We are now ready to state the main result by connecting Lemma~\ref{lem:compact_optimal_approx} and Theorem~\ref{thm:compact_eps_fosp} via Lemma~\ref{lem:residues_compact}. Let $ B = \|\tilde \calA\| \max\left\{ \left(\frac{\eps}{2\mu \bz \lambda_{\max}(A_0)}\right)^2, \frac{1}{\lambda_{\min}(A_0)^2} \left(\frac{\|C\|_2 + 3\sG\sqrt{n}}{2\mu} + \frac{3}{2}\bz \lambda_{\max}(A_0)\right) + \frac{\|\vb\|_2}{2\lambda_{\min}(A_0)} \right\} + \|\tilde \vb\|_2.$\footnote{We pick $\sG$ first and then $\eps$,  $B$ and $k$.} 
\begin{theorem}[Global optimality.]\label{thm:optimal_approx_compact} Let $\tilde X$ be a global optimum of \eqref{eq:penalty_sdp_compact}. Let $\delta \in (0, 1)$ and $\const$ be a universal constant. Draw a random matrix $G$ with $G_{ij} \sim \calN(0, \sigma_G^2)$ i.i.d.\ for $i \leq j$ and $G = G^T$. Let $U \in \Rnk$ be an $(\eps, \gamma)$-SOSP of~\eqref{eq:penalty_factored_compact} with perturbed cost matrix $C+G$ and:
\begin{align*}
	\eps \leq \left(\frac{\gamma k^2 \sigma_G^2 }{ 32\const n  \mu \|\calA\|^2}\right)^{\sfrac{2}{3}} ~\text{ and }~k \geq 3\left[\log\left(\frac{n}{\delta}\right) + \sqrt{ \rank(\calA)   \log\left( 1 + \frac{8 \mu B\|\tilde \calA\| \sqrt{\const  n}}{\sG } \right) }  \right].
\end{align*}
Then, with probability at least $1-O(\delta)$ the optimality gap obeys:
\begin{align*}
	\tilde F_\mu(UU^T) - \tilde F_\mu(\tilde X) & \leq \gamma \sqrt{\eps} \operatorname{Tr}(\tilde X) + \frac{1}{2}\eps \|U\|_F.
\end{align*}
\end{theorem}
This result shows that for compact SDPs~\eqref{eq:sdpcompact}, for $k =\tilde{\Omega}(\sqrt{m})$, all approximate SOSPs of the perturbed factorized problem are approximately globally optimal.

Notice that the result requires $\eps$ smaller than $\sG$, which is limiting but unavoidable as there can be SDPs with bad approximate SOSPs. Hence, if we perturb by only a small amount (small $\sG$), then we need to find highly accurate SOSPs to avoid these bad approximate SOSPs. Another way to look at the result is to see $\sG$ as a tentative distance from bad SDPs. Hence, for SDPs far away from these bad problems (higher $\sG$), even high $\eps$ solutions are approximately ($\eps$-) optimal.

\subsection{SDPs with positive definite cost}\label{sec:pd}
We now consider a second class of SDPs: ones where the cost matrix $C$ is positive definite. The feasible set of these SDPs need not be compact. However, FOSPs for these SDPs are bounded, hence we will be able to show similar results as in Section~\ref{sec:compact}. Consider the penalty formulation of the perturbed problem,
\begin{align}
\underset{U \in \R^{n \times k}}{\text{minimize}} ~ \widehat L_{\mu}(U) =  \ip{C+\tG}{UU^T}+\mu \sum_{i=1}^m \left(\ip{A_i}{UU^T} - b_i\right)^2,
\label{eq:smoothed}
\end{align}
where $\tG$ is a symmetric random matrix with $G_{ij} \stackrel{i.i.d.}{\sim} \mathcal{N}(0,\sG^2)$ for $i\leq j$. Let $\widehat F_{\mu}(UU^T) = \widehat L_{\mu} (U)$. To prove an optimality result for this problem, we first show a residue bound for any $\eps$-FOSP of $\widehat L_{\mu}(U)$.
\begin{lemma}\label{lem:residues}
Consider~\eqref{eq:smoothed} with a positive definite cost matrix $C$. Let $\sG \leq \frac{\lambda_{\min}(C)}{6\sqrt{n \log(n/ \delta)}}$. Then, with probability at least $1-\delta$, at any $\eps$-FOSP $U$ of~\eqref{eq:smoothed}, the residue obeys:
$$
	\|\vr\|_2 =\|\calA(UU^T) -\vb\|_2 \leq \|\calA\| \max \left \{ \left( \frac{2\eps} {\lambda_{\min}(C)}\right)^2, \frac{2\mu}{\lambda_{\min}(C)} \|\vb\|_2^2 \right \}+\|\vb\|_2.
$$ 
\end{lemma}

Using this, we get the following result from Lemma~\ref{lem:compact_optimal_approx} and Theorem \ref{thm:compact_eps_fosp}  along same lines as that of Theorem \ref{thm:optimal_approx_compact}.
Let $$B \triangleq \|\calA\| \max \left \{ \left( \frac{2\eps} {\lambda_{\min}(C)}\right)^2, \frac{2\mu}{\lambda_{\min}(C)} \|\vb\|_2^2 \right \}+\|\vb\|_2.$$
\begin{theorem}[Global optimality.]\label{thm:optimal_approx}
Let $\delta \in (0, 1)$ and $\const$ be a universal constant. Given an SDP \eqref{eq:sdp} with positive definite objective matrix $C$, let $\tilde X $ be a global optimum of the perturbed problem \eqref{eq:smoothed}, and let $\sG \leq \frac{\lambda_{\min}(C)}{4\sqrt{n \log(n/ \delta)}}$. Let $U$ be an $(\eps, \gamma)$-SOSP of the perturbed problem \eqref{eq:smoothed} with:
\begin{align*}
	\eps \leq \left(\frac{\gamma k^2 \sigma_G^2 }{ 32\const n  \mu \|\calA\|^2}\right)^{\sfrac{2}{3}} ~\text{ and }~ k \geq 3\left[\log\left(\frac{n}{\delta}\right) + \sqrt{ \rank(\calA)   \log \left(1 +\frac{8 \mu B\|\calA\|\sqrt{\const n} }{\sG }  \right) }  \right].
\end{align*}
Then, with probability at least $1-O(\delta)$,
\begin{align*}
	\widehat F_{\mu}(UU^T)  - \widehat F_{\mu}(\tilde X)  \leq \gamma \sqrt{\epsilon} \operatorname{Tr}(\tilde X) +\frac{1}{2}\eps \|U\|_F.
\end{align*}
\end{theorem}

This result shows that even though the feasible set of SDP is not compact, as long as the objective is positive definite, all approximate SOSPs of the perturbed objective are approximately optimal. Without the positive definite condition, SDPs can have unbounded solutions (see Section 2.4 of \citet{gartner2012approximation}). We also require a bound on the magnitude of the perturbation ($\sG$), as otherwise the objective ($C+G$) can be indefinite with (too) high probability, which may result in unbounded solutions.

\section{Applications}\label{sec:applications}

In this section, we present applications of our results to two SDPs: Max-Cut and matrix completion, both of which are important problems in the learning domain and have been studied extensively. Interest has grown to develop efficient solvers for these SDPs~\citep{arora2007combinatorial, pmlr-v65-mei17a, hardt2013understanding, bandeira2016low}.

This work differs from previous efforts in at least two ways. First, we aim to demonstrate that Burer--Monteiro-style approaches, which are often used in practice, can indeed lead to provably efficient algorithms for general SDPs. We believe that building upon this work, it should be possible to improve the time-complexity guarantees of such factorization-based algorithms. Second, we note that several problems formulated as SDPs in fact necessitate low-rank solutions, for example because of memory concerns (as is the case in matrix completion),  and factorization approaches provide a natural means to control rank. 

\subsection{Max-Cut}

We first consider the popular Max-Cut problem which finds applications in clustering related problems. In a seminal paper, \cite{goemans1995improved} defined the following SDP to solve the Max-Cut problem: $\min_{X\in \Rnn} \ip{C}{X}, \mbox{s.t. } X_{ii} = 1 \; \forall \; 1 \leq i \leq n, X \succeq 0 $, where $n$ is the number of vertices in the given graph and $C$ is its adjacency matrix. Since the constraint set also satisfies $\trace{X}=n$, we consider the following penalized, non-convex version of the problem.
\begin{align}
	\widehat{L}_{\mu}(U) \defeq \ip{C+G}{U\trans{U}} + \mu\left(\left(\ip{I}{U\trans{U}}-n\right)^2 + \sum_{i=1}^{n}\left(\ip{e_i \trans{e_i}}{U\trans{U}}-1\right)^2\right),\label{eqn:maxcut}
\end{align}
where $G$ is a random symmetric Gaussian matrix.  Let $\widehat F_{\mu}(UU^T) = \widehat L_{\mu} (U)$. After some simplifying computations, we have the following corollary of Theorem~\ref{thm:optimal_approx_compact}.
\begin{corollary}\label{cor:maxcut}
There exists an absolute numerical constant $c_1$ such that the following holds. With probability greater than $1-\delta$,
every $(\eps, \gamma)$-SOSP $U$ of the perturbed Max-Cut problem $\widehat{L}_{\mu}(U)$~\eqref{eqn:maxcut} with:
\begin{align*}\epsilon \leq \frac{1}{c_1} \left(\frac{\gamma \sigma_G^2}{\mu n}\right)^{2/3},~~ \text{ and } ~~  k = \tilde{\Omega} \left( \sqrt{n \log\left(\frac{\mu^2 \sqrt{n}}{\sigma_G}\right)}\right),
\end{align*}
satisfies $	\widehat{F}_{\mu}(UU^T) - \widehat{F}_{\mu}(X^*) \leq \gamma \sqrt{\epsilon} \trace{X^*} +\frac{1}{2} \epsilon \frob{U}$, where $X^*$ is a global optimum of $\widehat{F}_{\mu}(X)$.
\end{corollary}
The above result states that for the penalized version of the perturbed Max-Cut SDP, the Burer--Monteiro approach finds an approximate global optimum as soon as the factorization rank $k = \tilde{\Omega}(\sqrt{n})$. Existing results for Max-Cut using this approach either only handle exact SOSPs~\citep{boumal2016non}, or require $k=n+1$~\citep{boumal2016globalrates}, or require $k$ that is dependent on $\frac{1}{\eps}$~\citep{pmlr-v65-mei17a}. Moreover, complexity per iteration scales only linearly with the number of edges in the graph. 

\subsection{Matrix Completion}
In this section we specialize our results for the matrix completion problem \cite{candes2009exact}. The goal of a matrix completion problem is to find a low-rank matrix $M$ using only a small number of its entries, with applications in recommender systems. To ensure that the computed matrix is low-rank and generalizes well, one typically imposes nuclear-norm regularization which leads to the following SDP: 

\begin{minipage}{0.2\linewidth}
	\begin{align*}
	\min &\quad \trace{W_1} + \trace{W_2}\\ \text{s. t. }&\quad X_{ij} =M_{ij}, (i,j) \in \calS \\  &\quad \begin{bmatrix}W_1 & X \\ X^T & W_2\end{bmatrix} \succeq 0.
	\end{align*}
\end{minipage}
\begin{minipage}{0.05\linewidth}
	\begin{align*}
		\equiv \\
	\end{align*} \break
\end{minipage}
\begin{minipage}{0.6\linewidth}
	\begin{align*}
	\min & \quad \ip{I}{Z} \nonumber \\ \text{s. t. }&\quad \frac{1}{2}\ip{e_{i+n}e_{j+n}^T + e_{j+n} e_{i+n}^T}{Z} = M_{ij}, (i,j) \in \calS \nonumber \\  &\quad Z \succeq 0.
	\end{align*}
\end{minipage}
\noindent Here $\calS$ is the set of observed indices of $M$ and $Z\defeq \begin{bmatrix}W_1 & X \\ X^T & W_2\end{bmatrix}$. 
Let
\begin{align}
	\widehat{L}_{\mu}(U) = \ip{I+G}{UU^T} + \mu \sum_{i=1}^m \left(\frac{1}{2}\ip{e_{i+n}e_{j+n}^T + e_{j+n} e_{i+n}^T}{UU^T} - M_{ij} \right)^2  \label{eq:matcomp}
\end{align}
be the corresponding penalty objective.  Let $\widehat F_{\mu}(UU^T) = \widehat L_{\mu} (U)$. The objective is positive definite with $\lambda_1(C)=\lambda_n(C)=1$. Also, since $\calA$ is a sub-sampling operator, $\|\calA\| \leq 1$. Finally, for $\eps^2 \leq \frac{\mu}{2}\sqrt{\sum_{(i,j) \in \calS} M_{ij}^2}$, the residues are bounded by: \begin{align*} B&=\|\calA\| \max \left \{ \left( \frac{2\eps} {\lambda_n(C)}\right)^2, \frac{2\mu}{\lambda_n(C)} \|\vb\|_2^2 \right \}+\|\vb\|_2 \leq \max  3\mu \sqrt{\sum_{(i,j) \in \calS} M_{ij}^2}. \end{align*}

\noindent Applying Theorem~\ref{thm:optimal_approx} for this setting gives the following corollary.
\begin{corollary}\label{cor:mc_optimal}There exists an absolute numerical constant $c_2$ such that the following holds. With probability greater than $1-\delta$,
every $(\eps, \gamma)$-SOSP $U$ of the perturbed matrix completion problem $\widehat{L}_{\mu}(U)$~\eqref{eq:matcomp} with:
\begin{align*}\sG \leq \frac{1}{4\sqrt{n \log(n/ \delta)}},~~ \eps \leq \frac{1}{c_2}\left(\frac{\gamma \abs{\calS} \sigma_G^2 }{ n  \mu }\right)^{\sfrac{2}{3}}, ~\text{ and }~ k = \tilde{\Omega} \left( \sqrt{ \abs{\calS}   \log\left(\frac{\mu^2 \sqrt{n} \sqrt{\sum_{(i,j) \in \calS} M_{ij}^2}}{\sigma_G}\right) } \right),\end{align*} satisfies $\widehat{F}_{\mu}(UU^T)  - \widehat F_{\mu}(X^*)  \leq \gamma \sqrt{\epsilon} \trace{X^*} + \frac{1}{2} \eps \|U\|_F$, where $X^*$ is a global optimum of $\widehat{F}_{\mu}(X)$.
\end{corollary}
\noindent This result shows that for the  matrix completion problem with $m$ observations, for rank $\tilde{\Omega}(\sqrt{m})$, any approximate local minimum of the factorized and penalized problem is an approximate global minimum. 

Most of the existing results on matrix completion either require strong distribution assumptions on $\calS$ and incoherence assumptions on $M$ to recover a low-rank solution \citep{candes2009exact, jain2013low}. The standard nuclear norm minimization algorithms are not guaranteed to converge to low-rank solutions without these assumptions,  which implies that the entire matrix would need to be stored for prediction which is infeasible in practice. Similarly,  generalization error bounds \citep{foygel2011concentration} as well as differential privacy guarantees  depend on recovery of a low-rank solution.

Our result guarantees finding a rank -$\tilde{\Omega}(\sqrt{m})$ solution without any statistical assumptions on the sampling or the matrix. The tradeoff is our results do not guarantee finding a lower (potentially a constant) rank solution, even if one exists for a given problem.

%
%
%
%
%
%

\section{Gradient Descent}\label{sec:gd}
In previous sections we have seen that for the perturbed penalty objective~\eqref{eq:smoothed}, under some technical conditions on the SDP, with high probability upon appropriate choice of the parameters, every approximate SOSP is approximately optimal. Second-order methods such as cubic regularization and trust regions~\citep{nesterov2006cubic,cartis2012complexity} converge to an approximate SOSP in polynomial time.
While gradient descent with random initialization can take exponential time to converge to an SOSP~\citep{du2017gradient}, a recent line of work starting with~\citet{ge2015escaping} has established that perturbed gradient descent (PGD)\footnote{This is vanilla gradient descent but with additional random noise added to the updates when the gradient magnitude becomes smaller than a threshold.} converges to an SOSP as efficiently as second-order methods in the worst case, with high probability. In particular we have the following \emph{almost dimension free} convergence rate for PGD from~\citep{jin2017escape}.

\begin{theorem}[Theorem 3 of \citet{jin2017escape}]
Let $f$ be $l$-smooth (that is, its gradient is $l$-Lipschitz) and have a $\rho$-Lipschitz Hessian. There exists an absolute constant $c_{\max}$ such that, for any $\delta \in (0, 1)$, $\eps \leq \frac{l^2}{\rho}$, $\Delta_f \geq f(X_0) -f^*$, and constant $c \leq c_{\max}$, $PGD(X_0,l,\rho,\eps,c,\delta,\Delta_f)$ applied to the cost function $f$ outputs a $(\rho^2,\eps)$ SOSP with probability at least $1-\delta$ in
$$O \left( \frac{(f(X_0)-f^*)l}{\eps^2} \log^4 \left( \frac{nkl \Delta_f}{\eps^2 \delta} \right) \right)$$
iterations.
\end{theorem}

The above theorem requires the function $f$ to be smooth and Hessian-Lipschitz. The next lemma states that the perturbed penalty objective~\eqref{eq:smoothed} satisfies these requirements---proof in Appendix~\ref{apdx:gd}. 
\begin{lemma}\label{lem:gd_param}
In the region $\{U \in \Rnk : \|U\|_F \leq \tau \}$ for some $\tau > 0$, the cost function $\hat L_\mu(U)$ in~\eqref{eq:smoothed} is $l$-smooth and its Hessian is $\rho$-Lipschitz with:
\begin{itemize} 
	\item $l \leq 2\|C+\tG\|_2 + 4 \mu \|\calA\| \|\vb\|_2 + 12 \mu \tau^2 \|\calA\|^2$, and
	\item $\rho \leq 16\mu \tau \|\calA\|^2$.
\end{itemize}
Here, $\|\calA\|$ is as defined in~\eqref{eq:normofA}. Notice furthermore that, with high probability, $\|G\|_2 \leq 3\sG\sqrt{n}$. In that event, $\|C+G\|_2 \leq \|C\|_2 + 3\sG\sqrt{n}$.
\end{lemma}

Combining this lemma with the above theorem shows that the perturbed gradient method converges to an $(\eps, \rho^2)$ SOSP in $\widetilde{O}(\frac{1}{\eps^2})$ steps (ignoring all other problem parameters). This can be improved to $\widetilde{O}(\frac{1}{\eps^{1.75}})$ using a variant of Nesterov's accelerated gradient descent~\citep{jin2017accelerated}. 
Moreover, if the objective function is (restricted) strongly convex in the vicinity of the local minimum, then we can further improve the rates to $\textrm{poly} \log\left(\frac{1}{\epsilon}\right)$~\citep{jin2017escape}. This property is satisfied for problems where $\calA$ meets either restricted isometry conditions or when $\calA$ pertains to a uniform sampling of incoherent matrices~\citep{agarwal2010fast, negahban2012restricted,sun2014guaranteed}. See~\citep{bhojanapalli2016dropping} for more discussions on restricted strong convexity close to the global optimum.

The complexity of the algorithm is given by Gradient-Computation-Time $\times $ Number of iterations. Computing the gradient in each iteration requires $O\left(Zk+nk^2 + mnk \right)$ arithmetic operations where $Z$ is the number of non-zeros in $C$ and the constraint matrices. For dense problems this becomes $O\left( mn^2k \right)$. However, most practical problems tend to have a certain degree of sparsity in the constraint matrices so that the computational complexity of such a method can be significantly smaller than the worst-case bound.

\section{Conclusions and perspectives}\label{sec:conc}

In this paper we considered the Burer--Monteiro factorization to solve SDPs~\eqref{eq:factored}. In addition to dimensionality reduction, one advantage of such formulations is that algorithms for them necessarily produce positive semidefinite solutions of rank at most $k$. An ideal theorem would state that some polynomial-time algorithm computes approximate optima for~\eqref{eq:factored} in all cases with reasonably small $k$. In this regard, we now review what we achieved, what seems impossible and what remains to be done.

Because problem~\eqref{eq:factored} has nonlinear constraints, our first step was to move to a penalized formulation~\eqref{eq:penalty_factored}. For simplicity, we chose to work with a quadratic penalty. Quadratic penalties may require pushing $\mu$ to infinity to achieve constraint satisfaction at the optimum. Taking $\mu$ large may prove challenging numerically. In practice, it is known that augmented Lagrangian formulations (ALM) behave better in this respect~\citep{birgin2014ALM}. Thus, a first direction of improvement for the present work is to tackle ALM formulations instead.

Second, we established in Section~\ref{sec:exact} that for almost all SDPs, all exact SOSPs of~\eqref{eq:penalty_factored} are global optima which map to global optima of the penalized SDP~\eqref{eq:penalty_sdp} provided $\frac{k(k+1)}{2} > m$, where $m$ is the number of constraints. It should not be possible to improve the dependence on $k$ by much since certain SDPs admit a unique solution of rank $r$ such that $\frac{r(r+1)}{2} = m$. We showed in Section~\ref{sec:exact_sub} that for certain SDPs the penalty formulation~\eqref{eq:penalty_factored} admits suboptimal SOSPs. This suggests that even in the ideal statement stated above one may need to exclude some SDPs.

Third, we showed in Section~\ref{sec:smooth} that upon perturbing the cost matrix $C$ randomly (to avoid pathological cases), with high probability and provided $k = \tilde \Omega(\sqrt{m})$ (which is the right order though constants and dependence on other parameters could certainly be improved), when SOSPs have bounded residues (which is the case for positive definite cost matrices and for compact SDPs up to a technical modification), all SOSPs of the factored, penalized and perturbed problem are approximately optimal for that problem. This is achieved through smoothed analysis, which we believe is an appropriate tool to deal with the pathological cases exhibited above. These results can be further improved by deducing approximate constraint satisfaction and optimality for the original SDP~\eqref{eq:sdp}---which we currently do not do---and by further relaxing conditions on the SDP.

Finally, we studied the applicability of our results to two applications: Max-Cut and matrix completion. While these particularizations do not always improve over the specialized solvers for these problems, we believe that the work done here in studying low-rank parameterization of SDPs will be a helpful step towards building up to faster methods.

\section*{Acknowledgment}

NB thanks Dustin Mixon for many interesting conversations on the applicability of smoothed analysis to low-rank SDPs. NB was supported in part by NSF award DMS-1719558.

\bibliographystyle{abbrvnat}
\bibliography{sdp}

\begin{thebibliography}{58}
\providecommand{\natexlab}[1]{#1}
\providecommand{\url}[1]{\texttt{#1}}
\expandafter\ifx\csname urlstyle\endcsname\relax
  \providecommand{\doi}[1]{doi: #1}\else
  \providecommand{\doi}{doi: \begingroup \urlstyle{rm}\Url}\fi

\bibitem[Abbe(2017)]{abbe2017community}
E.~Abbe.
\newblock Community detection and stochastic block models: recent developments.
\newblock \emph{arXiv preprint arXiv:1703.10146}, 2017.

\bibitem[Agarwal et~al.(2010)Agarwal, Negahban, and
  Wainwright]{agarwal2010fast}
A.~Agarwal, S.~Negahban, and M.~J. Wainwright.
\newblock Fast global convergence rates of gradient methods for
  high-dimensional statistical recovery.
\newblock In \emph{Advances in Neural Information Processing Systems}, pages
  37--45, 2010.

\bibitem[Alizadeh(1995)]{alizadeh1995interior}
F.~Alizadeh.
\newblock Interior point methods in semidefinite programming with applications
  to combinatorial optimization.
\newblock \emph{SIAM Journal on Optimization}, 5\penalty0 (1):\penalty0 13--51,
  1995.

\bibitem[Anstreicher(2000)]{anstreicher2000volumetric}
K.~M. Anstreicher.
\newblock The volumetric barrier for semidefinite programming.
\newblock \emph{Mathematics of Operations Research}, 25\penalty0 (3):\penalty0
  365--380, 2000.

\bibitem[Arora and Kale(2007)]{arora2007combinatorial}
S.~Arora and S.~Kale.
\newblock A combinatorial, primal-dual approach to semidefinite programs.
\newblock In \emph{Proceedings of the thirty-ninth annual ACM symposium on
  Theory of computing}, pages 227--236. ACM, 2007.

\bibitem[Arora et~al.(2005)Arora, Hazan, and Kale]{arora2005fast}
S.~Arora, E.~Hazan, and S.~Kale.
\newblock Fast algorithms for approximate semidefinite programming using the
  multiplicative weights update method.
\newblock In \emph{Foundations of Computer Science, 2005. FOCS 2005. 46th
  Annual IEEE Symposium on}, pages 339--348. IEEE, 2005.

\bibitem[Bandeira et~al.(2016)Bandeira, Boumal, and
  Voroninski]{bandeira2016low}
A.~S. Bandeira, N.~Boumal, and V.~Voroninski.
\newblock On the low-rank approach for semidefinite programs arising in
  synchronization and community detection.
\newblock \emph{arXiv preprint arXiv:1602.04426}, 2016.

\bibitem[Barak et~al.(2015)Barak, Kelner, and Steurer]{barak2015dictionary}
B.~Barak, J.~A. Kelner, and D.~Steurer.
\newblock Dictionary learning and tensor decomposition via the sum-of-squares
  method.
\newblock In \emph{Proceedings of the forty-seventh annual ACM symposium on
  Theory of computing}, pages 143--151. ACM, 2015.

\bibitem[Barvinok(1995)]{barvinok1995problems}
A.~Barvinok.
\newblock Problems of distance geometry and convex properties of quadratic
  maps.
\newblock \emph{Discrete \& Computational Geometry}, 13\penalty0 (1):\penalty0
  189--202, 1995.
\newblock \doi{10.1007/BF02574037}.

\bibitem[Bhatia(2013)]{bhatia2013matrix}
R.~Bhatia.
\newblock \emph{Matrix analysis}, volume 169.
\newblock Springer Science \& Business Media, 2013.

\bibitem[Bhojanapalli et~al.(2016{\natexlab{a}})Bhojanapalli, Kyrillidis, and
  Sanghavi]{bhojanapalli2016dropping}
S.~Bhojanapalli, A.~Kyrillidis, and S.~Sanghavi.
\newblock Dropping convexity for faster semi-definite optimization.
\newblock In \emph{Conference on Learning Theory}, pages 530--582,
  2016{\natexlab{a}}.

\bibitem[Bhojanapalli et~al.(2016{\natexlab{b}})Bhojanapalli, Neyshabur, and
  Srebro]{bhojanapalli2016global}
S.~Bhojanapalli, B.~Neyshabur, and N.~Srebro.
\newblock Global optimality of local search for low rank matrix recovery.
\newblock In \emph{Advances in Neural Information Processing Systems}, pages
  3873--3881, 2016{\natexlab{b}}.

\bibitem[Birgin and Martínez(2014)]{birgin2014ALM}
E.~Birgin and J.~Martínez.
\newblock \emph{Practical Augmented {L}agrangian Methods for Constrained
  Optimization}.
\newblock Society for Industrial and Applied Mathematics, Philadelphia, PA,
  2014.
\newblock \doi{10.1137/1.9781611973365}.

\bibitem[Boumal et~al.(2016{\natexlab{a}})Boumal, Absil, and
  Cartis]{boumal2016globalrates}
N.~Boumal, P.-A. Absil, and C.~Cartis.
\newblock Global rates of convergence for nonconvex optimization on manifolds.
\newblock \emph{arXiv preprint arXiv:1605.08101}, 2016{\natexlab{a}}.

\bibitem[Boumal et~al.(2016{\natexlab{b}})Boumal, Voroninski, and
  Bandeira]{boumal2016non}
N.~Boumal, V.~Voroninski, and A.~Bandeira.
\newblock The non-convex burer-monteiro approach works on smooth semidefinite
  programs.
\newblock In \emph{Advances in Neural Information Processing Systems}, pages
  2757--2765, 2016{\natexlab{b}}.

\bibitem[Burer and Monteiro(2003)]{burer2003nonlinear}
S.~Burer and R.~D. Monteiro.
\newblock A nonlinear programming algorithm for solving semidefinite programs
  via low-rank factorization.
\newblock \emph{Mathematical Programming}, 95\penalty0 (2):\penalty0 329--357,
  2003.

\bibitem[Burer and Monteiro(2005)]{burer2005local}
S.~Burer and R.~D. Monteiro.
\newblock Local minima and convergence in low-rank semidefinite programming.
\newblock \emph{Mathematical Programming}, 103\penalty0 (3):\penalty0 427--444,
  2005.

\bibitem[Cand{\`e}s and Recht(2009)]{candes2009exact}
E.~J. Cand{\`e}s and B.~Recht.
\newblock Exact matrix completion via convex optimization.
\newblock \emph{Foundations of Computational mathematics}, 9\penalty0
  (6):\penalty0 717--772, 2009.

\bibitem[Cartis et~al.(2012)Cartis, Gould, and Toint]{cartis2012complexity}
C.~Cartis, N.~I. Gould, and P.~L. Toint.
\newblock Complexity bounds for second-order optimality in unconstrained
  optimization.
\newblock \emph{Journal of Complexity}, 28\penalty0 (1):\penalty0 93--108,
  2012.

\bibitem[Du et~al.(2017)Du, Jin, Lee, Jordan, Singh, and
  Poczos]{du2017gradient}
S.~S. Du, C.~Jin, J.~D. Lee, M.~I. Jordan, A.~Singh, and B.~Poczos.
\newblock Gradient descent can take exponential time to escape saddle points.
\newblock In \emph{Advances in Neural Information Processing Systems}, pages
  1067--1077, 2017.

\bibitem[Foygel and Srebro(2011)]{foygel2011concentration}
R.~Foygel and N.~Srebro.
\newblock Concentration-based guarantees for low-rank matrix reconstruction.
\newblock In \emph{Proceedings of the 24th Annual Conference on Learning
  Theory}, pages 315--340, 2011.

\bibitem[Garber(2016)]{garber2016faster}
D.~Garber.
\newblock Faster projection-free convex optimization over the spectrahedron.
\newblock In \emph{Advances in Neural Information Processing Systems}, pages
  874--882, 2016.

\bibitem[Garber and Hazan(2016)]{garber2016sublinear}
D.~Garber and E.~Hazan.
\newblock Sublinear time algorithms for approximate semidefinite programming.
\newblock \emph{Mathematical Programming}, 158\penalty0 (1-2):\penalty0
  329--361, 2016.

\bibitem[G{\"a}rtner and Matousek(2012)]{gartner2012approximation}
B.~G{\"a}rtner and J.~Matousek.
\newblock \emph{Approximation algorithms and semidefinite programming}.
\newblock Springer Science \& Business Media, 2012.

\bibitem[Ge and Ma(2017)]{ge2017optimization}
R.~Ge and T.~Ma.
\newblock On the optimization landscape of tensor decompositions.
\newblock In \emph{Advances in Neural Information Processing Systems}, pages
  3656--3666, 2017.

\bibitem[Ge et~al.(2015)Ge, Huang, Jin, and Yuan]{ge2015escaping}
R.~Ge, F.~Huang, C.~Jin, and Y.~Yuan.
\newblock Escaping from saddle points---online stochastic gradient for tensor
  decomposition.
\newblock In \emph{Proceedings of The 28th Conference on Learning Theory},
  pages 797--842, 2015.

\bibitem[Ge et~al.(2016)Ge, Lee, and Ma]{ge2016matrix}
R.~Ge, J.~D. Lee, and T.~Ma.
\newblock Matrix completion has no spurious local minimum.
\newblock In \emph{Advances in Neural Information Processing Systems}, pages
  2973--2981, 2016.

\bibitem[Ge et~al.(2017)Ge, Jin, and Zheng]{ge2017no}
R.~Ge, C.~Jin, and Y.~Zheng.
\newblock No spurious local minima in nonconvex low rank problems: A unified
  geometric analysis.
\newblock In \emph{International Conference on Machine Learning}, pages
  1233--1242, 2017.

\bibitem[Goemans and Williamson(1995)]{goemans1995improved}
M.~X. Goemans and D.~P. Williamson.
\newblock Improved approximation algorithms for maximum cut and satisfiability
  problems using semidefinite programming.
\newblock \emph{Journal of the ACM (JACM)}, 42\penalty0 (6):\penalty0
  1115--1145, 1995.

\bibitem[Hardt(2013)]{hardt2013understanding}
M.~Hardt.
\newblock Understanding alternating minimization for matrix completion.
\newblock \emph{arXiv preprint arXiv:1312.0925}, 2013.

\bibitem[Hazan(2008)]{hazan2008sparse}
E.~Hazan.
\newblock Sparse approximate solutions to semidefinite programs.
\newblock In \emph{LATIN 2008: Theoretical Informatics}, pages 306--316.
  Springer, 2008.

\bibitem[Helmke and Shayman(1995)]{helmke1995matrixlsq}
U.~Helmke and M.~Shayman.
\newblock Critical points of matrix least squares distance functions.
\newblock \emph{Linear Algebra and its Applications}, 215:\penalty0 1--19,
  1995.
\newblock \doi{https://doi.org/10.1016/0024-3795(93)00070-G}.

\bibitem[Jain et~al.(2013)Jain, Netrapalli, and Sanghavi]{jain2013low}
P.~Jain, P.~Netrapalli, and S.~Sanghavi.
\newblock Low-rank matrix completion using alternating minimization.
\newblock In \emph{Proceedings of the 45th annual ACM Symposium on theory of
  computing}, pages 665--674. ACM, 2013.

\bibitem[Jin et~al.(2017{\natexlab{a}})Jin, Ge, Netrapalli, Kakade, and
  Jordan]{jin2017escape}
C.~Jin, R.~Ge, P.~Netrapalli, S.~M. Kakade, and M.~I. Jordan.
\newblock How to escape saddle points efficiently.
\newblock \emph{arXiv preprint arXiv:1703.00887}, 2017{\natexlab{a}}.

\bibitem[Jin et~al.(2017{\natexlab{b}})Jin, Netrapalli, and
  Jordan]{jin2017accelerated}
C.~Jin, P.~Netrapalli, and M.~I. Jordan.
\newblock Accelerated gradient descent escapes saddle points faster than
  gradient descent.
\newblock \emph{arXiv preprint arXiv:1711.10456}, 2017{\natexlab{b}}.

\bibitem[Journ{\'e}e et~al.(2010)Journ{\'e}e, Bach, Absil, and
  Sepulchre]{journee2010low}
M.~Journ{\'e}e, F.~Bach, P.-A. Absil, and R.~Sepulchre.
\newblock Low-rank optimization on the cone of positive semidefinite matrices.
\newblock \emph{SIAM Journal on Optimization}, 20\penalty0 (5):\penalty0
  2327--2351, 2010.

\bibitem[Keshavan et~al.(2010)Keshavan, Montanari, and Oh]{keshavan2010matrix}
R.~H. Keshavan, A.~Montanari, and S.~Oh.
\newblock Matrix completion from a few entries.
\newblock \emph{Information Theory, IEEE Transactions on}, 56\penalty0
  (6):\penalty0 2980--2998, 2010.

\bibitem[Krishnan and Mitchell(2003)]{krishnan2003properties}
K.~Krishnan and J.~E. Mitchell.
\newblock Properties of a cutting plane method for semidefinite programming.
\newblock 2003.

\bibitem[Lanckriet et~al.(2004)Lanckriet, Cristianini, Bartlett, Ghaoui, and
  Jordan]{lanckriet2004learning}
G.~R. Lanckriet, N.~Cristianini, P.~Bartlett, L.~E. Ghaoui, and M.~I. Jordan.
\newblock Learning the kernel matrix with semidefinite programming.
\newblock \emph{Journal of Machine learning research}, 5\penalty0
  (Jan):\penalty0 27--72, 2004.

\bibitem[Mei et~al.(2017)Mei, Misiakiewicz, Montanari, and
  Oliveira]{pmlr-v65-mei17a}
S.~Mei, T.~Misiakiewicz, A.~Montanari, and R.~I. Oliveira.
\newblock Solving sdps for synchronization and maxcut problems via the
  grothendieck inequality.
\newblock In S.~Kale and O.~Shamir, editors, \emph{Proceedings of the 2017
  Conference on Learning Theory}, volume~65 of \emph{Proceedings of Machine
  Learning Research}, pages 1476--1515, Amsterdam, Netherlands, 07--10 Jul
  2017. PMLR.

\bibitem[Natarajan(1995)]{natarajan1995sparse}
B.~K. Natarajan.
\newblock Sparse approximate solutions to linear systems.
\newblock \emph{SIAM journal on computing}, 24\penalty0 (2):\penalty0 227--234,
  1995.

\bibitem[Negahban and Wainwright(2012)]{negahban2012restricted}
S.~Negahban and M.~J. Wainwright.
\newblock Restricted strong convexity and weighted matrix completion: Optimal
  bounds with noise.
\newblock \emph{The Journal of Machine Learning Research}, 13\penalty0
  (1):\penalty0 1665--1697, 2012.

\bibitem[Nesterov and Nemirovski(1988)]{nesterov1988polynomial}
Y.~Nesterov and A.~Nemirovski.
\newblock Polynomial barrier methods in convex programming.
\newblock \emph{Ekonom. i Mat. Metody}, 24\penalty0 (6):\penalty0 1084--1091,
  1988.

\bibitem[Nesterov and Nemirovski(1989)]{nesterov1989self}
Y.~Nesterov and A.~Nemirovski.
\newblock \emph{Self-concordant functions and polynomial-time methods in convex
  programming}.
\newblock USSR Academy of Sciences, Central Economic \& Mathematic Institute,
  1989.

\bibitem[Nesterov and Polyak(2006)]{nesterov2006cubic}
Y.~Nesterov and B.~T. Polyak.
\newblock Cubic regularization of newton method and its global performance.
\newblock \emph{Mathematical Programming}, 108\penalty0 (1):\penalty0 177--205,
  2006.

\bibitem[Nesterov et~al.(1994)Nesterov, Nemirovski, and
  Ye]{nesterov1994interior}
Y.~Nesterov, A.~Nemirovski, and Y.~Ye.
\newblock \emph{Interior-point polynomial algorithms in convex programming},
  volume~13.
\newblock SIAM, 1994.

\bibitem[Nguyen(2017)]{nguyen2017repulsion}
H.~H. Nguyen.
\newblock Random matrices: repulsion in spectrum.
\newblock \emph{arXiv preprint arXiv:1709.06682}, 2017.

\bibitem[Park et~al.(2017)Park, Kyrillidis, Carmanis, and
  Sanghavi]{park2017non}
D.~Park, A.~Kyrillidis, C.~Carmanis, and S.~Sanghavi.
\newblock Non-square matrix sensing without spurious local minima via the
  burer-monteiro approach.
\newblock In \emph{Artificial Intelligence and Statistics}, pages 65--74, 2017.

\bibitem[Pataki(1998)]{pataki1998rank}
G.~Pataki.
\newblock On the rank of extreme matrices in semidefinite programs and the
  multiplicity of optimal eigenvalues.
\newblock \emph{Mathematics of operations research}, 23\penalty0 (2):\penalty0
  339--358, 1998.
\newblock \doi{10.1287/moor.23.2.339}.

\bibitem[Recht et~al.(2010)Recht, Fazel, and Parrilo]{recht2010guaranteed}
B.~Recht, M.~Fazel, and P.~A. Parrilo.
\newblock Guaranteed minimum-rank solutions of linear matrix equations via
  nuclear norm minimization.
\newblock \emph{SIAM review}, 52\penalty0 (3):\penalty0 471--501, 2010.

\bibitem[Shi and Malik(2000)]{shi2000normalized}
J.~Shi and J.~Malik.
\newblock Normalized cuts and image segmentation.
\newblock \emph{IEEE Transactions on pattern analysis and machine
  intelligence}, 22\penalty0 (8):\penalty0 888--905, 2000.

\bibitem[Sun et~al.(2016)Sun, Qu, and Wright]{sun2016geometric}
J.~Sun, Q.~Qu, and J.~Wright.
\newblock A geometric analysis of phase retrieval.
\newblock \emph{preprint arXiv:1602.06664}, 2016.

\bibitem[Sun and Luo(2014)]{sun2014guaranteed}
R.~Sun and Z.-Q. Luo.
\newblock Guaranteed matrix completion via non-convex factorization.
\newblock \emph{arXiv preprint arXiv:1411.8003}, 2014.

\bibitem[Vandenberghe and Boyd(1996)]{vandenberghe1996semidefinite}
L.~Vandenberghe and S.~Boyd.
\newblock Semidefinite programming.
\newblock \emph{SIAM review}, 38\penalty0 (1):\penalty0 49--95, 1996.

\bibitem[Vershynin(2016)]{vershynin2016high}
R.~Vershynin.
\newblock High-dimensional probability, 2016.

\bibitem[Wolkowicz(1981)]{wolkowicz1981optimization}
H.~Wolkowicz.
\newblock Some applications of optimization in matrix theory.
\newblock \emph{Linear algebra and its applications}, 40:\penalty0 101--118,
  1981.

\bibitem[Yurtsever et~al.(2017)Yurtsever, Udell, Tropp, and
  Cevher]{yurtsever2017sketchy}
A.~Yurtsever, M.~Udell, J.~Tropp, and V.~Cevher.
\newblock Sketchy decisions: Convex low-rank matrix optimization with optimal
  storage.
\newblock In \emph{Artificial Intelligence and Statistics}, pages 1188--1196,
  2017.

\bibitem[Zhu et~al.(2017)Zhu, Li, Tang, and Wakin]{zhu2017global}
Z.~Zhu, Q.~Li, G.~Tang, and M.~B. Wakin.
\newblock Global optimality in low-rank matrix optimization.
\newblock \emph{arXiv preprint arXiv:1702.07945}, 2017.

\end{thebibliography}

\clearpage
\appendix

\section{Proof of Lemma~\ref{lem:eigenvalue_main}: lower-bound for smallest singular values}\label{apdx:proofNguyen}

First we state a special case of Corollary 1.17
from~\citep{nguyen2017repulsion}. Let $N_I(X) $, denote the number of
eigenvalues of $X$ in the interval $I$.
\begin{corollary}\label{cor:Nguyen}
  Let $M'$ be a deterministic symmetric matrix in $\Snn$. Let $G'$ be a
  random symmetric matrix with entries $G'_{ij}$ drawn i.i.d.\ from
  $\N(0,1)$ for
  $i \geq j$ (in particular, independent of $M'$.) Then, for given $0 < \gamma < 1$, there exists a
  constant $c = c(\gamma)$ such that for any $\eps > 0$ and $k \geq 1$,
  with $I$ being the interval,
  $[-\frac{\eps k}{\sqrt{n}}, \frac{\eps k}{\sqrt{n}}]$,
$$
\Pr{N_I(M'+G') \geq k } \leq n^k \left(\frac{c\eps}{\sqrt{2\pi}}\right)^{(1-\gamma)k^2/2}.
$$
\end{corollary}
We can use the above corollary to prove
Lemma~\ref{lem:eigenvalue_main}.
\begin{proof} 
	In our case, entries of $G$ have variance $\sG^2$. Thus, set $G = \sG G'$, and set $\bar M = \sG M'$.
  From Corollary~\ref{cor:Nguyen}, we get
  \begin{align*}
	  N_{\sG I}(\bar M+G) = N_I(M'+G') < k
  \end{align*}
  with probability
  at least $1 - n^k \left(\frac{c\eps}{\sqrt{2\pi}}\right)^{(1-\gamma)k^2/2}$. In this event,
  $\sigma_{n-(k-1)}(\bar M+G) \geq \frac{\eps k}{\sqrt{n}} \sG$. Choose
  $\gamma =\frac{1}{2}$, and $\eps = \frac{1}{2 c}$. Substituting this we get with
  probability at least
  $1 - \exp\left( - \frac{k^2}{8} \log( 8 \pi) + k \log (n)\right)$
  that
$$
\sigma_{n-(k-1)}(\bar M+G) \geq \frac{k}{2c \sqrt{n}}\sG.
$$
Hence,
$\sum_{i=1}^{k} \sigma_{n-(i-1)}\left(\bar M+G\right)^2 \geq
\sigma_{n-(k-1)}\left(\bar M+G\right)^2 \geq \frac{k^2}{\const n}
 \sG^2$, for some absolute constant $\const = 4c^2$.
\end{proof}
\section{Proofs for Section~\ref{sec:exact}}\label{app:exact}

\begin{proof}[Proof of Lemma~\ref{lem:global}]
	Necessary and sufficient optimality conditions for~\eqref{eq:prob_fx} are: $\nabla f(X) \succeq 0$ and $\nabla f(X)X = 0$. Let $U$ be an SOSP for~\eqref{eq:prob_fx_U} with $\rank(U) < k$  and define $X = UU^T$. Then, $\nabla g(U) = 2\nabla f(UU^T)U = 0$ and $\nabla^2 g(U) \succeq 0$. The first statement readily shows that $\nabla f(X)X = 0$. The Hessians of $f$ and $g$ are related by:
	\begin{align*}
		\frac{1}{2}\nabla^2 g(U)[\dot U] & = \nabla f(UU^T)\dot U + \nabla^2 f(UU^T)[U\dot U^T + \dot U U^T]U.
	\end{align*}
	Since $\rank(U) < k$, there exists a vector $z \in \Rk$ such that $Uz = 0$ and $\|z\|_2 = 1$. For any $x \in \Rn$, set $\dot U = xz^T$ so that $U\dot U^T + \dot U U^T = 0$. Using second-order stationarity of $U$, we find:
	\begin{align*}
		0 \leq \frac{1}{2}\ip{\dot U}{\nabla^2 g(U)[\dot U]} & = \ip{xz^T}{\nabla f(UU^T)xz^T} = x^T \nabla f(UU^T) x.
	\end{align*}
	This holds for all $x \in \Rn$, hence $\nabla f(UU^T) \succeq 0$ and $X = UU^T$ is optimal for~\eqref{eq:prob_fx}. Since~\eqref{eq:prob_fx} is a relaxation of~\eqref{eq:prob_fx_U}, it follows that $U$ is optimal for~\eqref{eq:prob_fx_U}.
\end{proof}

\begin{proof}[Proof of Lemma~\ref{lem:rank_deficient}]
	Let $U$ be any FOSP of~\eqref{eq:penalty_factored} and consider the linear operator $\calA \colon \Snn \to \Rm$ defined by $\calA(X)_i = \ip{A_i}{X}$. By first-order stationarity, we have: 
	\begin{align*}
		\nabla L_\mu(U) & = 2\left( C + 2\mu \calA^*(\calA(UU^T) - b) \right)U = 0.
	\end{align*}
	Hence, the nullity of $C + 2\mu \calA^*(\calA(UU^T) - b)$ (the dimension of its kernel) satisfies:
	\begin{align}
		\rank(U) \leq \nulll(C + 2\mu \calA^*(\calA(UU^T) - b)) \leq \max_{y \in \Rm} \nulll(C + \calA^*(y)).
		\label{eq:maxnulliny}
	\end{align}
	The maximum over $y$ is indeed attained since the function $\nulll$ takes integer values in $0, \ldots, n$. Say the maximum evaluates to $\ell$. Then, for some $y$, $M \triangleq C + \calA^*(y)$ has nullity $\ell$. Hence,
	\begin{align*}
		C & = M - \calA^*(y) \in \calN_\ell + \im \calA^*,
	\end{align*}
	where $\calN_\ell$ is the manifold of symmetric matrices of size $n$ and nullity $\ell$, $\im \calA^*$ is the range of $\calA^*$ and the plus is a set-sum. More generally, assuming the maximum in~\eqref{eq:maxnulliny} is $p$ or more, then
	\begin{align*}
		C & \in \calM_p \triangleq \bigcup_{\ell = p, \ldots, n} \calN_\ell + \im \calA^*.
	\end{align*}
	The manifold $\calN_\ell$ has dimension $\frac{n(n+1)}{2} - \frac{\ell(\ell+1)}{2}$~\citep[Prop.~2.1(i)]{helmke1995matrixlsq}, while $\im \calA^*$ has dimension at most $m$. Hence, $$\dim \calM_p \leq m + \max_{\ell = p, \ldots, n} \dim \calN_\ell = m + \frac{n(n+1)}{2} - \frac{p(p+1)}{2}.$$ Since $C$ is in $\Snn$ and $\dim \Snn = \frac{n(n+1)}{2}$, almost no $C$ lives in $\calM_p$ if $\dim\calM_p < \dim \Snn$, which is the case if $\frac{p(p+1)}{2} > m$. Stated differently: $\rank(U) \leq p$, and for almost all $C \in \Snn$, $\frac{p(p+1)}{2} \leq m$. To conclude, require that $k$ is strictly larger than any $p$ which satisfies $\frac{p(p+1)}{2} \leq m$.
\end{proof}

\section{Proofs for Section~\ref{sec:exact_sub}}\label{app:exact_sub}
\begin{proof}[Proof of Theorem~\ref{thm:bad_sdp}]
	We first show that the SDP admits exactly one feasible point. Indeed, let $X \succeq 0$ be feasible for the SDP. Then, constraints $n$ and $n+1$ imply $\ip{A_{n+1} - A_{n}}{X} = 0$. That is, the trace of the principal submatrix of size $n-1$ of $X$ is zero. Since this submatrix is also positive semidefinite, it is zero. Constraints 1 to $n-1$ further show that all entries but $X_{nn}$ are zero. Finally, constraints $n$ and ${n+1}$ force $X_{nn} = \frac{5(n-1)}{3}$. This $X$ has rank~1 and is necessarily optimal.
	
	We now show that the proposed $\bar U$ is suboptimal for $L$. To this end, build $\tilde U \in \Rnk$ with the last row having squared 2-norm equal to $\frac{5(n-1)}{3}$, and all other rows are zero. Clearly, $\tilde U\tilde U^T$ is feasible for the SDP, so that $L(\tilde U) = 0$: this is optimal. On the other hand, $L(\bar U) = \frac{5}{18}(n-1)^2\eps^2 > L(\tilde U)$.
	
	Finally, we check stationarity of $\bar U$. Let $\calA \colon \Snn \to \R^m$ be the linear operator such that $\calA(X)_i = \ip{A_i}{X}$, and define the residue function $\vr(U) = \calA(UU^T) - b$. The cost function and its derivatives take the following forms:
	\begin{align*}
		L(U) & = \frac{1}{2} \|\vr(U)\|_2^2, \\
		\nabla L(U) & = 2\calA^*(\vr(U))U, \\
		\nabla^2 L(U)[\dot U] & = 2\calA^*(\vr(U))\dot U + 2\calA^*(\calA(U\dot U^T + \dot U U^T))U.
	\end{align*}
	Simple computations show that $\calA(\bar U \bar U^T) = (0, \ldots, 0, (n-1)\eps, 2(n-1)\eps)^T$, so that $\calA^*(\vr(\bar U)) = -\frac{n-1}{3} \eps^2 \cdot e_n^{} e_n^T$: only the bottom-right entry is non-zero. Consequently, $\nabla L(\bar U) = 0$: $\bar U$ is an FOSP To show second-order stationarity, we must also show that $\nabla^2 L(\bar U)$ is positive semidefinite.
	That is, we must show the inequalities:
	\begin{align*}
		0 \leq \ip{\dot U}{\nabla^2 L(U)[\dot U]} & = 2\ip{\dot U \dot U^T}{\calA^*(\vr(U))} + \left\|\calA(U \dot U^T + \dot U U^T)\right\|_2^2
	\end{align*}
	for all $\dot U \in \Rnk$. Let
	\begin{align*}
		\dot U & = \begin{bmatrix}
			\textrm{---} &\dot u_1^T & \textrm{---} \\ & \vdots&  \\ \textrm{---} &\dot u_n^T& \textrm{---}
		\end{bmatrix}, \quad \textrm{ with } \quad \dot u_1, \ldots, \dot u_n \in \Rk \textrm{ arbitrary.}
	\end{align*}
	Then, $\calA(\bar U\dot U^T + \dot U\bar U^T) = (2\dot u_n^T, q_1, q_2)^T$ for some values $q_1, q_2$, so that:
	\begin{align*}
		\ip{\dot U}{\nabla^2 L(\bar U)[\dot U]} & = -2\frac{n-1}{3} \eps^2 \|\dot u_n\|_2^2 + 4 \|\dot u_n\|_2^2 + q_1^2 + q_2^2 \geq \left(4-2\frac{n-1}{3} \eps^2 \right) \|\dot u_n\|_2^2.
	\end{align*}
	Under our condition on $\eps$, this is indeed always nonnegative: $\bar U$ is an SOSP.
\end{proof}

\section{Proofs for Section~\ref{sec:smooth}}\label{app:smooth}

\begin{proof}[Proof of Lemma \ref{lem:compact_optimal_approx}]
	The gradient and Hessian of $L_\mu$~\eqref{eq:penalty_factored}, with $\vr \triangleq \vr(U) = \calA(UU^T)-b$, are:
	\begin{align}
		\nabla L_\mu(U) & = 2\left( C + 2\mu\calA^*(\vr) \right)U,\label{eq:gradLmu}\\
		\nabla^2 L_\mu(U)[\dot U] & = 2\left( C + 2\mu\calA^*(\vr) \right)\dot U + 4\mu\calA^*(\calA(\dot U U^T + U\dot U^T))U. \label{eq:HessianLmu}
	\end{align}
	Since $U$ is an $(\eps, \gamma)$-SOSP, it holds for all $\dot U \in \Rnk$ with $\|\dot U\|_F = 1$ that:
	\begin{align}
		-\frac{\gamma\sqrt{\eps}}{2} & \leq \frac{1}{2} \ip{\dot U}{\nabla^2 L_\mu(U)[\dot U]} = \ip{C + 2\mu\calA^*(\vr)}{\dot U \dot U^T} + \mu \left\|\calA(\dot U U^T + U \dot U^T)\right\|_2^2.
		 \label{eq:HessianLmuip}
	\end{align}
	We now construct specific $\dot U$'s to exploit the fact that $U$ is almost rank deficient. Let $z \in \R^k$ be a right singular vector of $U$ such that $\|Uz\|_2 =\sigma_k(U)$ (that is, $z$ is associated to the least singular value of $U$ and $\|z\|_2 = 1$.) For any $x \in \Rn$ with $\|x\|_2 = 1$, introduce $\dot U = xz^T$ in~\eqref{eq:HessianLmuip}:
	\begin{align*}
	 	-\frac{\gamma\sqrt{\eps}}{2} & \leq x^T(C + 2\mu\calA^*(\vr))x + \mu \left\|\calA(\dot U U^T + U \dot U^T)\right\|_2^2.
	\end{align*}
	The last term is easily controlled:
	\begin{align*}
		\left\|\calA(\dot U U^T + U \dot U^T)\right\|_2 \leq 2 \|\calA\| \|U\dot U^T\|_F = 2 \|\calA\| \|Uzx^T\|_F \leq 2 \|\calA\|\|Uz\|_2 \|x\|_2 = 2 \|\calA\|\sigma_k(U).
	\end{align*}
	Let $x$ be an eigenvector of $C + 2\mu\calA^*(\vr)$ associated to its least eigenvalue and combine the last two statements together with the assumption on $\sigma_k(U)$ to find:
	\begin{align}
		\lambda_{\min}(C + 2\mu\calA^*(\vr)) \geq -\frac{\gamma\sqrt{\eps}}{2} - 4\mu\|\calA\|^2\sigma_k^2(U) \geq -\gamma\sqrt{\eps}.
		\label{eq:lambdaminC2mu}
	\end{align}
	This inequality is key to bound the optimality gap. For this part, we rely on the fact that $L_\mu(U) = F_\mu(UU^T)$ and $F_\mu$ is convex on $\Snn$~\eqref{eq:penalty_sdp}. Specifically, let $\tilde X$ be a global optimum for $F_\mu$ (assuming it exists), and set $X = UU^T$. Then, $\nabla F_\mu(X) = C + 2\mu\calA^*(\vr), \nabla L_\mu(U) = 2\nabla F_\mu(X)U$ and:
	\begin{align*}
		F_\mu(\tilde X) - F_\mu(X) & \geq \ip{\nabla F_\mu(X)}{\tilde X - X}  = \ip{C+2\mu\calA^*(\vr)}{\tilde X} - \frac{1}{2}\ip{\nabla L_\mu(U)}{U} \\
								   & \geq -\gamma \sqrt{\eps} \operatorname{Tr}(\tilde X) - \frac{1}{2}\eps \|U\|_F.
	\end{align*}
	In the last step, we used~\eqref{eq:lambdaminC2mu} as well as approximate first-order stationarity .
%
\end{proof}

\begin{proof}[Proof of Proposition \ref{prop:compactSDPconsraintA0}]
	One direction is elementary: if there exists $A_0 \succ 0$ and $b_0 \geq 0$ such that $\ip{A_0}{X} = b_0$ for all $X \in \calC$, then,
	\begin{align*}
		\forall X \in \calC, \quad \trace{X} = \ip{I_n}{X} \leq \lambda_{\min}(A_0)^{-1} \ip{A_0}{X} = \lambda_{\min}(A_0)^{-1} b_0.
	\end{align*}
	Thus, the trace of $X \succeq 0$ is bounded, and it follows that $\calC$ is compact. Furthermore: if $b_0 = 0$, then $\calC = \{0\}$; and if $b_0 > 0$, then $0 \notin \calC$.
	
	To prove the other direction, assume $\calC$ is non-empty and compact. If $\calC = \{0\}$, let $A_0 = I_n, b_0 = 0$. Now assume $\calC \neq \{0\}$. The SDP comes in a primal-dual pair:
	\begin{align*}
	\min_{X\in\Snn} \ip{C}{X} \quad & \textrm{ s.t. } \quad \calA(X) = b, \ X \succeq 0, \tag{P} \label{eq:proofP}\\
	\max_{y\in\R^m} \ip{b}{y} \quad & \textrm{ s.t. } \quad C \succeq \calA^*(y). \tag{D} \label{eq:proofD}
	\end{align*}
	It is well known that if~\eqref{eq:proofD} is infeasible, then~\eqref{eq:proofP} is unbounded or infeasible~\citep[Thm.~4.1(a)]{wolkowicz1981optimization}. Since we assume $\calC$ is non-empty, this simplifies to: if~\eqref{eq:proofD} is infeasible, then~\eqref{eq:proofP} is unbounded. The contrapositive states: if~\eqref{eq:proofP} is bounded, then~\eqref{eq:proofD} is feasible. By our compactness assumption on $\calC$, we know that~\eqref{eq:proofP} is bounded for all $C \in \Snn$. Thus,~\eqref{eq:proofD} is feasible for any $C$. In particular, take $C = -I_n$: there exists $-y \in \R^m$ such that $A_0 \triangleq \calA^*(y) \succeq I_n$. Furthermore,
	\begin{align*}
		\forall X \in \calC, \quad \ip{A_0}{X} = \ip{\calA^*(y)}{X} = \ip{y}{\calA(X)} = \ip{y}{b} \triangleq b_0.
	\end{align*}
	Since there exists $X \neq 0$ in $\calC$, it follows that $b_0 > 0$.
\end{proof}

\begin{proof}[Proof of Theorem \ref{thm:compact_eps_fosp}]
	Using~\eqref{eq:gradLmu}, $U$ is an $\eps$-FOSP of the perturbed problem if and only if $\| (M+\tG)U\|_F \leq \frac{\eps}{2}$, where $M = C + 2\mu\calA^*(\vr)$.
	Let $U = P \Sigma Q^T$ be a thin SVD of $U$ ($P$ is $n\times k$ with orthonormal columns; $Q$ is $k\times k$ orthogonal). Then,
	\begin{align*}
	\| (M+\tG)U\|_F & = \| (M+\tG) P \Sigma \|_F \\
	& \geq \sigma_k(U)\| (M+\tG) P\|_F \\
	& \geq  \sigma_k(U) \sqrt{\sum_{i=1}^k \sigma_{n-(i-1)}(M+\tG)^2}.
	\end{align*}
	Hence, we control the smallest singular value of $U$ in terms of $\eps$ and the $k$ smallest singular values of $M+G$:
	\begin{align}
		\sigma_k(U) & \leq \frac{\eps}{2\sqrt{\sum_{i=1}^{k}\sigma_{n-(i-1)}(M+\tG)^2}}.
		\label{eq:keyboundsigmak}
	\end{align}
	The next lemma helps lower-bound the denominator---it follows from Theorem 1.16 and Corollary 1.17 in \citep{nguyen2017repulsion}; see proof in Appendix~\ref{apdx:proofNguyen}.
	\begin{lemma}\label{lem:eigenvalue_main}
		Let $\bar M$ be a fixed symmetric matrix of size $n$. Let $G$ be a symmetric Gaussian matrix of size $n$, independent of $\bar M$, with diagonal and upper-triangular entries sampled independently from $\N(0,\sG^2)$. There exists an absolute constant $\const$ such that:
		\begin{align*}
		\Pr{\sum_{i=1}^{k} \sigma_{n-(i-1)}\left(\bar M+G\right)^2 <  \frac{k^2}{\const n} \sG^2 } \leq \exp\left( - \frac{k^2}{8} \log(8 \pi) + k \log (n)\right).
		\end{align*}
	\end{lemma}
	We cannot use Lemma~\ref{lem:eigenvalue_main} directly, as in our case $M$ is not statistically independent of $G$. Indeed, $M$ depends on $U$ through the residue $\vr = \vr(U)$ and $U$ is an $\eps$-FOSP: a feature that depends on $G$.
	To resolve this, we cover the set of possible $M$'s with a net, under the assumption that $\vr$ is bounded. Lemma~\ref{lem:eigenvalue_main} provides a bound for each $\bar M$ in this net. This can be extended to hold for all $\bar M$'s in the net simultaneously via a union bound. By taking a sufficiently dense net, we can then infer that $M$ is necessarily close to one of these $\bar M$'s, and conclude.
	
	Let $\calE$ be the event (on $G$) that $\|\vr\|_2  \leq B$ for all $\eps$-FOSPs of the perturbed problem.
	Conditioned on $\calE$, we have
	\begin{align*}
		\|M - C\|_F & = 2\mu\|\calA^*(\vr)\|_F \leq 2\mu B \|\calA\|,
	\end{align*}
	where $\|\calA\|$ is defined in~\eqref{eq:normofA}.
	As a result, $M$ lies in a ball of center $C$ and radius $2 \mu B\|\calA\|$ in an affine subspace of dimension $\rank(\calA)$.  A unit-ball in Frobenius norm in $d$ dimensions admits an $\varepsilon$-net of $(1+2/\varepsilon)^d$ points~\citep[Cor.~4.2.13]{vershynin2016high}. Thus, we can pick a net with
	$\left( 1 + \frac{4 \mu  B\|\calA\|}{\sG} \sqrt{\frac{4\const n}{k^2}} \right)^{\rank(\calA)}$ points in such a way that, independently of $\vr$, there exists a point $\bar M$ in the net satisfying:
	\begin{align}
		\| \bar M - M \|_F \leq \sqrt{\frac{k^2}{4\const n}} \sG = \frac{k}{2\sqrt{\const n}}\sG.
		\label{eq:epscover}
	\end{align}
	Let $T \colon \Snn \to \Rk$ be defined by $T_q(A) = (\sigma_{n-q+1}(A), \ldots, \sigma_n(A))^T$, that is: $T$ extracts the $q$ smallest singular values of $A$, in order. Then,
	\begin{align*}
		\| \bar M - M \|_F & = \| (\bar M+G) - (M+G) \|_F \\
						   & \geq \| T_n(\bar M+G) - T_n(M + G) \|_2 \\
						   & \geq \| T_k(\bar M+G) - T_k(M + G) \|_2 \\
						   & \geq \| T_k(\bar M+G) \|_2 - \| T_k(M + G) \|_2,
	\end{align*}
	where the first inequality follows from~\citep[Ex.~IV.3.5]{bhatia2013matrix}. Hence,
	\begin{align}
		\sqrt{\sum_{i=1}^{k}\sigma_{n-(i-1)}(M+\tG)^2} \geq \sqrt{\sum_{i=1}^{k}\sigma_{n-(i-1)}(\bar M+\tG)^2} - \| \bar M - M \|_F. \label{eq:foo}
	\end{align}
	Now, taking a union bound for $\calE$ and for Lemma~\ref{lem:eigenvalue_main} over each $\bar M$ in the net, we get~\eqref{eq:epscover} and
	\begin{align}
		\sqrt{ \sum_{i=1}^{k} \sigma_{n-(i-1)}\left(\bar M+G\right)^2 } \geq \frac{k}{\sqrt{\const n}} \sG
		\label{eq:bar}
	\end{align}
	with probability at least
	$$
		1 - \exp \left( - \frac{k^2}{8}\log(8\pi) + k\log(n) + \rank(\calA) \cdot \log\left( 1 + \frac{4 \mu  B\|\calA\|}{\sG} \sqrt{\frac{4\const n}{k^2}} \right) \right) - \delta.
	$$
	Inside the log, we can safely replace $k$ with 1, as this only hurts the probability. Then, the result holds with probability at least
	$$
		1 - \exp \left( - \frac{k^2}{8}\log(8\pi) + k\log(n) + \rank(\calA) \cdot \log\left( 1 + \frac{8 \mu  B\|\calA\|}{\sG} \sqrt{\const n} \right) \right) - \delta.
	$$
	We aim to pick $k$ so as to ensure
	\begin{align*}
		\exp \left( - \frac{k^2}{8}\log(8\pi) + k\log(n) + \rank(\calA) \cdot \log\left( 1 + \frac{8 \mu  B\|\calA\|}{\sG} \sqrt{\const n} \right) \right) \leq \delta'.
	\end{align*}
	This is a quadratic condition of the form
	\begin{align*}
		-ak^2 + bk + c \leq \log(\delta')
	\end{align*}
	for some $a, b > 0$, $c \geq 0$. Since $k$ is positive we get, $k \geq \frac{b+\sqrt{a (c+\log(1/\delta'))}}{a}$, which is satisfied for, $$ k \geq  3\left[\log\left(\frac{n}{\delta'}\right) + \sqrt{ \rank(\calA)   \log\left( 1 + \frac{8 \mu B\|\calA\| \sqrt{\const  n}}{\sG } \right) }  \right].$$

	Combining~\eqref{eq:keyboundsigmak},~\eqref{eq:epscover},~\eqref{eq:foo} and~\eqref{eq:bar}, we find:
	\begin{align*}
		\sigma_k(U) \leq \frac{ \epsilon}{\sG } \frac{2\sqrt{\const n}}{k}
	\end{align*}
	 with probability at least $1-\delta-\delta'$.
%
%
\end{proof}

\begin{proof}[Proof of Lemma \ref{lem:residues_compact}]
	If $U = 0$, the bounds clearly hold: assume $U \neq 0$ in what follows.
	Using $\nabla \tilde L_\mu(U) = 2( C + 2\mu\tilde \calA^*(\tilde \vr) )U$, the definition of $\eps$-FOSP reads:
	\begin{align*}
		\frac{\eps}{2} \geq \left\| \left( C + 2\mu\tilde \calA^*(\tilde \vr) \right) U \right\|_F.
	\end{align*}
	Combining this with $\|A\|_F \geq \frac{1}{\|B\|_F} \ip{A}{B}$ for $B\neq 0$ (Cauchy--Schwarz) gives:
	\begin{align*}
		\frac{\eps}{2} \geq \frac{1}{\|U\|_F} \ip{\left( C + 2\mu\tilde \calA^*(\tilde \vr) \right) U}{U}.
	\end{align*}
	This can be further developed as:
	\begin{align}
		\frac{\eps \|U\|_F}{2} & \geq \ip{C + 2\mu\tilde \calA^*(\tilde \vr)}{UU^T} \nonumber\\
		& = \ip{C}{UU^T} + 2\mu \ip{\tilde\vr}{\tilde\calA(UU^T)} \nonumber\\
		& = \ip{C}{UU^T} + 2\mu \ip{\tilde \calA(UU^T) - \tilde \vb}{\tilde \calA(UU^T)}. \label{eq:residueboundintermediate}
	\end{align}
	At this point, we separate the constraint $(A_0, b_0)$ from the rest, using the usual definition for $(\calA, \vb)$ which capture constraints $1, \ldots, m$:
	\begin{align*}
		\frac{\eps \|U\|_F}{2} & \geq \ip{C}{UU^T}+2\mu\left( \ip{\calA(UU^T)-\vb}{\calA(UU^T)}+ \left(\ip{A_0}{UU^T}-b_0\right)\ip{A_0}{UU^T} \right) \\
		& \geq \ip{C}{UU^T} + 2\mu\left( \|\calA(UU^T)\|_2^2-\|\vb\|_2 \|\calA(UU^T)\|_2 + \left(\ip{A_0}{UU^T}-b_0\right)\ip{A_0}{UU^T} \right).
	\end{align*}
	Let $y =\|\calA(UU^T)\|_2$. Then the above inequality holds when
	$$
	y^2 -\|\vb\|_2 y +\frac{1}{2\mu}\left( \ip{C}{UU^T}-\frac{\eps \|U\|_F}{2} \right)+  \left(\ip{A_0}{UU^T}-b_0\right)\ip{A_0}{UU^T}  \leq 0.
	$$
	For this to happen we need the above quadratic to have real roots. This requires:
	\begin{align*}
	\frac{1}{4}\|\vb\|_2^2 & \geq \frac{1}{2\mu}\left( \ip{C}{UU^T}-\frac{\eps \|U\|_F}{2} \right)+  (\ip{A_0}{UU^T}-\bz)\ip{A_0}{UU^T} \\
	& \geq\frac{1}{2\mu}\left( -\|CU\|_F \|U\|_F-\frac{\eps \|U\|_F}{2} \right)+  \lambda_{\min}(A_0)^2 \|U\|_F^4-\bz \lambda_{\max}(A_0) \|U\|_F^2 \\
	& \geq \lambda_{\min}(A_0)^2 \|U\|_F^4 -\frac{\|C\|_2}{2\mu}\|U\|_F^2 - \bz \lambda_{\max}(A_0) \|U\|_F^2 - \frac{\eps}{4\mu}\|U\|_F,
	\end{align*}
	where we used that for any two matrices $A$ and $B$, it holds that $\|AB\|_F \leq \|A\|_2 \|B\|_F$.
	Focus on the last two terms of the last inequality. We distinguish two cases. Either
	\begin{align*}
		\bz \lambda_{\max}(A_0) \|U\|_F^2 + \frac{\eps}{4\mu}\|U\|_F \geq  \frac{3}{2}\bz \lambda_{\max}(A_0)\|U\|_F^2,
	\end{align*}
	in which case $\|U\|_F \leq \frac{\eps}{2\mu \bz \lambda_{\max}(A_0)}$ (assuming $b_0 > 0$). Or the opposite holds, and:
	\begin{align*}
		\frac{1}{4} \|\vb\|_2^2 & \geq \lambda_{\min}(A_0)^2 \|U\|_F^4 - \left(\frac{\|C\|_2}{2\mu} + \frac{3}{2}\bz \lambda_{\max}(A_0)\right)\|U\|_F^2.
	\end{align*}
	This is a quadratic inequality in $y = \|U\|_F^2$ of the form $ay^2 - by - c \leq 0$ with coefficients $a> 0$ and $b, c \geq 0$. Such a quadratic always has at least one real root, so that $y \leq \frac{b + \sqrt{b^2 + 4ac}}{2a}$. Furthermore, $\sqrt{b^2 + 4ac} \leq \sqrt{b^2 + (\sqrt{4ac})^2 + 2b\sqrt{4ac}} = b + \sqrt{4ac}$. Hence, $y \leq \frac{b}{a} + \sqrt{\frac{c}{a}}$, which means:
	\begin{align*}
		\|U\|_F^2 & \leq \frac{1}{\lambda_{\min}(A_0)^2} \left(\frac{\|C\|_2}{2\mu} + \frac{3}{2}\bz \lambda_{\max}(A_0)\right) + \frac{\|\vb\|_2}{2\lambda_{\min}(A_0)}.
	\end{align*}
	Accounting for the two distinguished cases, we find:
	\begin{align*}
		\|U\|_F^2 & \leq \max\left\{ \left(\frac{\eps}{2\mu \bz \lambda_{\max}(A_0)}\right)^2, \frac{1}{\lambda_{\min}(A_0)^2} \left(\frac{\|C\|_2}{2\mu} + \frac{3}{2}\bz \lambda_{\max}(A_0)\right) + \frac{\|\vb\|_2}{2\lambda_{\min}(A_0)} \right\}.
	\end{align*}
	
	We now bound the residues (generically) in terms of $\|U\|_F$, using submultiplicativity for $\|UU^T\|_F \leq \|U\|_F^2$ and the definition of $\|\calA\|$~\eqref{eq:normofA}:
	\begin{align*}
		\|\vr\|_2 = \| \calA(UU^T) - \vb \|_2 \leq \|\calA\| \|UU^T\|_F  + \|\vb\|_2 \leq \|\calA\| \|U\|_F^2 + \|\vb\|_2.
	\end{align*}
	Evidently, the same bound holds for $\tilde\calA, \tilde \vb, \tilde \vr$.
\end{proof}

\begin{proof}[Proof of Theorem \ref{thm:optimal_approx_compact}]
	
	By Lemma~\ref{lem:residues_compact}, for a problem perturbed with $G$, the residues of all $\eps$-FOSPs, $\|\tilde \vr\|_2$, are bounded as:
	\begin{align*}
		\|\tilde \calA\| \max\left\{ \left(\frac{\eps}{2\mu \bz \lambda_{\max}(A_0)}\right)^2, \frac{1}{\lambda_{\min}(A_0)^2} \left(\frac{\|C+G\|_2}{2\mu} + \frac{3}{2}\bz \lambda_{\max}(A_0)\right) + \frac{\|\vb\|_2}{2\lambda_{\min}(A_0)} \right\} + \|\tilde \vb\|_2
	\end{align*}
	With probability at least $1 - \delta$, $\|C + G\|_2 \leq \|C\|_2 + 3\sG\left(\sqrt{n} + \sqrt{2\log(1/\delta)}\right)$. Hence, Theorem~\ref{thm:compact_eps_fosp} applies with this $\delta$ and
	\begin{align*}
		B = \|\tilde \calA\| \max\left\{ \left(\frac{\eps}{2\mu \bz \lambda_{\max}(A_0)}\right)^2, \frac{1}{\lambda_{\min}(A_0)^2} \left(\frac{\|C\|_2 + 3\sG\sqrt{n}}{2\mu} + \frac{3}{2}\bz \lambda_{\max}(A_0)\right) + \frac{\|\vb\|_2}{2\lambda_{\min}(A_0)} \right\} + \|\tilde \vb\|_2.
	\end{align*}
	Hence, with $k$ as prescribed in that theorem for a given $\delta' =\delta \in (0, 1)$, with probability at least $1 - 2\delta$, it holds that
	\begin{align*}
		\sigma_k(U) \leq \frac{ 2\epsilon}{\sG } \frac{\sqrt{\const n}}{k}
	\end{align*}
	for any $\eps$-FOSP. Lemma~\ref{lem:compact_optimal_approx} requires $\sigma_k^2(U) \leq \frac{\gamma \sqrt{\eps}}{8 \mu\|\calA\|^2}$. Hence, we choose: $\eps \leq \left(\frac{\gamma k^2 \sigma_G^2 }{ 32\const n  \mu \|\calA\|^2}\right)^{\sfrac{2}{3}}$, and with probability at least $1-2\delta$  hypothesis of Lemma \ref{lem:compact_optimal_approx} is satisfied. Let $\tilde X$ be a global optimum for $\tilde F_\mu$, then the optimality gap obeys:
\begin{align*}
\tilde F_\mu(UU^T) - \tilde F_\mu(\tilde X) & \leq \gamma \sqrt{\eps} \operatorname{Tr}(\tilde X) + \frac{1}{2}\eps \|U\|_F.
\end{align*}
\end{proof}

\subsection{Proof of section 4.2}
\begin{proof}[Proof of Lemma \ref{lem:residues}]
With probability at least $1-\delta$, $\sigma_1(G) \leq 3\sG\sqrt{n} $. In that event, for $\sG \leq \frac{\lambda_{n}(C)}{6\sqrt{n \log(n/\delta)}}$, we have $C+G \succeq \frac{\lambda_{n}(C)}{2}I$.

$U$ is an $\eps$-FOSP of \eqref{eq:smoothed} implies $\|2(C+\tG+2\mu \calA^*(\vr))U\|_F \leq \eps$. \begin{align*}
\frac{\eps}{2} &\geq \left\| \left(C+\tG+2\mu \calA^*(\vr) \right)U \right\|_F \\
&\geq \frac{1}{\|U\|_F}  \ip{C+\tG+2\mu \calA^*(\vr)}{UU^T} .
\end{align*} Hence, \begin{align*} 
\frac{\eps \|U\|_F}{2}  &\geq \ip{C+\tG}{UU^T} + 2\mu\ip{ \calA^*(\vr)}{UU^T}  \\
&\geq  \frac{\lambda_n(C)}{2}\|U\|_F^2+2\mu\ip{\vr}{\calA(UU^T)} \\
&\geq \frac{\lambda_n(C)}{2}\|U\|_F^2+2\mu( \|\calA(UU^T)\|_2^2 - \|\vb\|_2\|\calA(UU^T)\|_2).
 \end{align*}
The above inequality is  a quadratic in $y=\|\calA(UU^T)\|_2$: $y^2 -y \|\vb\|_2 + \frac{1}{2\mu} \left(\frac{\lambda_n(C)}{2}\|U\|_F^2 -\frac{\eps \|U\|_F}{2}\right) \leq 0$. If $\frac{\eps \|U\|_F}{2} \geq \frac{\lambda_n(C)}{4}\|U\|_F^2$, then  $\|U\| _F \leq \frac{2\eps} {\lambda_n(C)}$. Else, for the above inequality to hold we need the quadratic to have real roots.
\begin{align*}
\|\vb\|_2^2 &\geq 4 \cdot 1 \cdot \frac{1}{2\mu} \left( \frac{\lambda_n(C)}{2}\|U\|_F^2 -\frac{\eps \|U\|_F}{2} \right) \\
&\geq \frac{2}{\mu} \frac{\lambda_n(C)}{4}\|U\|_F^2.
\end{align*}
The last inequality follows from  $\frac{\eps \|U\|_F}{2} \leq \frac{\lambda_n(C)}{4}\|U\|_F^2$. Hence, $\|U\|_F^2 \leq \max \left \{ \left( \frac{2\eps} {\lambda_n(C)}\right)^2, \frac{2\mu}{\lambda_n(C)} \|\vb\|_2^2 \right \}$. Hence,
\begin{multline*}
\|\vr\|_2 = \| \calA(UU^T) -\vb\|_2 \leq \|\calA(UU^T)\|_2 +\|\vb\|_2 \leq \|\calA\| \|UU^T\|_F +\|\vb\|_2 \\ \leq \|\calA\| \max \left \{ \left( \frac{2\eps} {\lambda_n(C)}\right)^2, \frac{2\mu}{\lambda_n(C)} \|\vb\|_2^2 \right \}+\|\vb\|_2.
\end{multline*}
\end{proof}

\section{Proofs for Section~\ref{sec:gd}}\label{apdx:gd}

\begin{proof}[Proof of Lemma \ref{lem:gd_param}.]
We start by showing that the gradient is $l$-Lipschitz continuous. The gradient is given by:
   $$\nabla {\widehat L_\mu}(U) = \left[2(C+\tG) + 4\mu\calA^*(\vr)\right]U,$$
   where $\vr = \vr(U) = \calA(UU^T) - \vb$. Hence, for $U_1, U_2 \in \Rnk$, with notation $\vr_1 = \vr(U_1), \vr_2 = \vr(U_2)$,
  \begin{align*}
  \norm{ \nabla {\widehat L_\mu}(U_1) -\nabla {\widehat L_\mu}(U_2)}_F
  & \leq \norm{ 2(C+\tG)(U_1-U_2)}_F + 4\mu \norm{\calA^*(\vr_1)U_1 - \calA^*(\vr_2)U_2  }_F \\
  & \leq 2\|C+\tG\|_2 \| U_1 -U_2\|_F + 4 \mu \norm{\calA^*(\vr_1) (U_1- U_2)  }_F \\
  & \quad \quad + 4 \mu \norm{\calA^*(\vr_1 - \vr_2) U_2}_F \\ 
  & \leq \left(2\|C+\tG\|_2 + 4 \mu \left\| \calA^*(\vr_1) \right\|_2 \right) \| U_1 -U_2\|_F	\\
  & \quad \quad + 4 \mu \norm{\calA^*(\vr_1 - \vr_2) U_2}_F.
  \end{align*} 
 This further simplifies using the norm of $\calA$~\eqref{eq:normofA}: $\left\| \calA^*(\vr_1) \right\|_2 \leq \|\calA\| \|\vr_1\|_2$ and $\|\vr_1\|_2 \leq \|\calA\| \|U_1\|_F^2 + \|\vb\|_2$, so that if $\|U_1\|_F \leq \tau$:
 \begin{align*}
 \left\| \calA^*(\vr_1) \right\|_2 & \leq (\tau^2\|\calA\|+\|\vb\|_2)\|\calA\|.
 \end{align*}
 Similarly, using $\|U_2\|_F \leq \tau$ as well:
 \begin{align}
\norm{\calA^*(\vr_1 - \vr_2) U_2}_F & \leq \|\calA^*(\calA(U_1^{}U_1^T - U_2^{}U_2^T))\|_2 \|U_2\|_F\nonumber\\
& \leq \tau \|\calA\|^2 \|U_1^{}U_1^T - U_2^{}U_2^T\|_F \nonumber\\
& = \tau \|\calA\|^2 \|U_1^{}U_1^T - U_1^{}U_2^T + U_1^{}U_2^T - U_2^{}U_2^T\|_F \nonumber\\
& \leq \tau \|\calA\|^2 \left( \|U_1^{}(U_1 - U_2)^T \|_F + \|(U_1^{} - U_2^{})U_2^T\|_F \right) \nonumber\\
& \leq 2\tau^2 \|\calA\|^2 \|U_1 - U_2\|_F. \label{eq:Astarr1r2}
\end{align}
Combining, we find
\begin{align*}
\norm{ \nabla {\widehat L_\mu}(U_1) -\nabla {\widehat L_\mu}(U_2)}_F
& \leq \left(2\|C+\tG\|_2 + 4 \mu \|\calA\| (\tau^2\|\calA\|+\|\vb\|_2) \right) \| U_1 -U_2\|_F	\\
& \quad \quad + 8 \mu \tau^2 \|\calA\|^2 \|U_1 - U_2\|_F,
\end{align*}
which establishes the Lipschitz constant for $\nabla {\widehat L_\mu}$.

We now show that the Hessian is $\rho$-Lipschitz continuous in operator norm, that is, we must show that for any $U_1$ and $U_2$ with norms bounded by $\tau$,
\begin{align*}
	\underset{\|{\dot U}\|_F \leq 1}{\max}\ip{\nabla^2 {\widehat L_\mu}(U_1)[{\dot U}] -\nabla^2 {\widehat L_\mu}(U_2)[{\dot U}]}{{\dot U}} \leq \rho \|U_1 -U_2\|_F.
\end{align*}
Recall from~\eqref{eq:HessianLmuip} that
\begin{align*}
	\ip{\nabla^2 {\widehat L_\mu}(U)[{\dot U}]}{{\dot U}} = 2\ip{C+\tG+2\mu \calA^*(\vr)}{{\dot U}{\dot U}^T} + 2\mu \|\calA(U{\dot U}^T+{\dot U}U^T)\|_2^2.
\end{align*}
Hence,
\begin{multline*}
	\ip{\nabla^2 {\widehat L_\mu}(U_1)[{\dot U}]}{{\dot U}} -\ip{\nabla^2 {\widehat L_\mu}(U_2)[{\dot U}]}{{\dot U}}  \\
 = 4\mu\ip{\calA^*(\vr_1 - \vr_2)}{{\dot U}{\dot U}^T} +  2\mu \left(\|\calA(U_1{\dot U}^T+{\dot U}U_1^T)\|_2^2 - \|\calA(U_2{\dot U}^T+{\dot U}U_2^T)\|_2^2\right).
\end{multline*}
On one hand, following the same reasoning as in~\eqref{eq:Astarr1r2}, we have
\begin{align*}
	\ip{\calA^*(\vr_1 - \vr_2)}{{\dot U}{\dot U}^T} & \leq \|\calA^*(\vr_1 - \vr_2)\|_F \|{\dot U}{\dot U}^T\|_F \\
	& \leq 2\tau \|\calA\|^2 \|U_1 - U_2\|_F \|\dot U\|_F^2.
\end{align*}
On the other hand, using that for any two vectors $u, v$ we have
\begin{align*}
	\|u\|_2^2 - \|v\|_2^2 = \ip{u+v}{u-v} \leq \|u+v\|_2 \|u-v\|_2 \leq (\|u\|_2 + \|v\|_2)\|u - v\|_2,
\end{align*}
we can find:
\begin{align*}
\|\calA(U_1{\dot U}^T+{\dot U}U_1^T)\|_2^2 - \|\calA(U_2{\dot U}^T+{\dot U}U_2^T)\|_2^2 & \leq 4 \tau \|\calA\|^2 \|U_1 - U_2\|_F \|\dot U\|_F^2.
\end{align*}
For this, we used $\|\calA(U\dot U^T + \dot U U^T)\|_2 \leq \|\calA\| \|U\dot U^T + \dot U U^T\|_F \leq \tau \|\calA\|\|\dot U\|_F$ when $\|U\|_F \leq \tau$ and 
\begin{align*}
	\|\calA(U_1{\dot U}^T+{\dot U}U_1^T - U_2{\dot U}^T - {\dot U}U_2^T)\|_2
	 & \leq \|\calA\| \left( \|(U_1-U_2)\dot U^T\|_F + \|\dot U(U_1-U_2)^T\|_F \right) \\
	 & \leq 2\|\calA\|\|\dot U\|_F \|U_1 - U_2\|_F.
\end{align*}
Overall, this shows $\rho = 16\mu \tau \|\calA\|^2$ is an appropriate Lipschitz constant.
\end{proof}

\end{document}